%% file: main.tex
\documentclass{ecai_2025} 

\usepackage{graphicx}

\newcommand{\BibTeX}{B\kern-.05em{\sc i\kern-.025em b}\kern-.08em\TeX}

\usepackage[T1]{fontenc}
\usepackage{latexsym}

\usepackage{booktabs} 
\usepackage{nicefrac}       
\usepackage{amsmath,amssymb}
\usepackage{mathtools}
\usepackage{pifont}
\usepackage[makeroom]{cancel}
\usepackage{placeins}
\usepackage{algorithm} 
\usepackage{algorithmic}
\usepackage{subcaption}
\usepackage{xcolor}

\usepackage{pgf}
\usepackage{tikz}
\usetikzlibrary{shapes.geometric}
\usetikzlibrary{shapes.arrows}
\usetikzlibrary{shapes,arrows,automata,positioning,fit,arrows.meta,calc}
\usetikzlibrary{decorations.markings}
\usetikzlibrary{patterns}
\usetikzlibrary{3d}
\usetikzlibrary{quotes,matrix}
\usepackage{verbatim}
\definecolor{lightblue}{HTML}{1f77b4}

\definecolor{blue1}{HTML}{00A6FB}
\definecolor{green1}{HTML}{97DB4F}

\usepackage{pgfplots}
\usepgfplotslibrary{fillbetween}
\pgfplotsset{compat=newest, 
	tick label style={font=\scriptsize},
	label style={font=\scriptsize},
	legend style={font=\scriptsize}}

\newenvironment{customlegend}[1][]{%
	\begingroup
	\csname pgfplots@init@cleared@structures\endcsname
	\pgfplotsset{#1}%
}{%
	\csname pgfplots@createlegend\endcsname
	\endgroup
}%

\def\addlegendimage{\csname pgfplots@addlegendimage\endcsname}

\pgfplotscreateplotcyclelist{custom}{%
	green1!60!black,dashed,thick,every mark/.append style={solid},mark=triangle*\\
	orange,dotted,thick,every mark/.append style={solid},mark=square*\\
	blue1,thick,every mark/.append style={solid},mark=*\\
	violet!90!white,dash dot,ultra thick\\
	mikadoYellow,densely dashed, ultra thick\\
	lapisLazuli,densely dotted, ultra thick\\
	harvardCrimson,densely dash dot, ultra thick\\
	harlequin, ultra thick\\
	neonFuchsia,densely dotted, ultra thick\\
	bronze,densely dash dot, ultra thick\\
	black, ultra thick\\
}

\usepackage[capitalize,noabbrev]{cleveref}
\usepackage{enumitem}

\usepackage{amsthm}
\theoremstyle{remark}

\theoremstyle{definition}

\usepackage{thmtools, thm-restate}

\usepackage{xspace}
\DeclareRobustCommand{\eg}{e.g.,\@\xspace}
\DeclareRobustCommand{\ie}{i.e.,\@\xspace}
\DeclareRobustCommand{\wrt}{w.r.t.\@\xspace}

\DeclareMathOperator*{\E}{\mathbb{E}}

\DeclareMathOperator*{\ssup}{\sup}

\begin{document}


\begin{frontmatter}


\paperid{7434} 


\title{On the Sample Efficiency of Abstractions and Potential-Based Reward Shaping in Reinforcement Learning}

\author[1]{\fnms{Giuseppe}~\snm{Canonaco}\thanks{Corresponding Author. Email: giuseppe.canonaco@jpmorgan.com}}
\author[2]{\fnms{Leo}~\snm{Ardon}}
\author[1]{\fnms{Alberto}~\snm{Pozanco}}
\author[1]{\fnms{Daniel}~\snm{Borrajo}}

\address[1]{AI Research Dept. JPMorganChase Madrid, ES}
\address[2]{AI Research Dept. JPMorganChase London, UK}


\begin{abstract}
The use of Potential-Based Reward Shaping (PBRS) has shown great promise in the ongoing research effort to tackle sample inefficiency in Reinforcement Learning (RL). However, choosing the right potential function remains an open challenge. Additionally, RL techniques are usually constrained to use a finite horizon for computational limitations, which introduces a bias when using PBRS.
In this paper, we first build some theoretically-grounded intuition on why selecting the potential function as the optimal value function of the task at hand produces performance advantages. We then analyse the bias induced by finite horizons in the context of PBRS producing novel insights. Finally, leveraging abstractions as a way to approximate the optimal value function of the given task, we assess the sample efficiency and performance impact of PBRS on four environments including a goal-oriented navigation task and three Arcade Learning Environments (ALE) games. Remarkably, experimental results show that we can reach the same level of performance as CNN-based solutions with a simple fully-connected network.
\end{abstract}

\end{frontmatter}


\section{Introduction}\label{sec:intro}
Reinforcement Learning (RL) is able to achieve impressive results at the cost of requiring a huge amount of samples~\citep{silver2018general, vinyals2019grandmaster, openai2019dota}.
This sample inefficiency prevents RL from being more widely applied in the real world because of the cost associated with experience collection. Thus, the increasing demand of algorithms that can counteract or mitigate this issue.

In the RL literature, many different methods tackle the above-mentioned problem~\citep{taylor2009transfer, papini2018stochastic, touati2022does}.
However, a promising way to compensate for the widely known sample inefficiency of RL is through Potential-Based Reward Shaping (PBRS)~\citep{ng1999policy}.
One property of this technique is that if we select the potential function to be the optimal value function of the problem at hand, the resulting reshaped problem will have an optimal value function that is zero everywhere, thus simplifying the learning process~\citep{ng1999policy}.
Unfortunately, the literature lacks theoretically-grounded motivations explaining why PBRS should entail sample efficiency for such a choice in RL problems.

Additionally, the optimal value function for most tasks is not available in practice, so computing estimates of that function becomes the only viable option.
In case some knowledge about the task we are facing is available (\eg via subject-matter experts), one way to approximate its optimal value function is through abstractions~\citep{Cipollone2023}.
An abstraction here represents a mapping of the original task to a simpler one, such that solving the simpler one can accelerate the learning phase in the original task itself.

Due to computational limitations, we also have to practically impose a maximum number of steps allowed to solve an RL problem (finite horizon).
This practical consideration has the disadvantage of introducing a bias in the context of PBRS as highlighted by recent literature~\citep{eck2016potential, grzes2017reward, Cipollone2023}. 
Importantly, the introduced bias may disrupt the policy ordering in the reshaped Markov Decision Process (MDP)~\citep{puterman2014markov} practically turning policies that were optimal into sub-optimal ones. 
Therefore, the need for a theoretical analysis investigating what can be preserved when applying PBRS under finite horizons.

In this paper we try to cover all the above mentioned gaps in the literature. In order to do so, we provide:
\begin{itemize}
    \item A theoretically-grounded intuition on why PBRS entails performance benefits when we select the potential function as the optimal value function of the task at hand.
    \item A formal proof for the absence of bias when using PBRS in finite-horizon goal-oriented MDPs (under some conditions on the potential if the MDP is stochastic) addressing an issue opened by~\citet{Cipollone2023}.
    \item A novel investigation of the above-mentioned bias for general MDPs, producing conditions under which we can preserve the total policy ordering (or only the optimal policy) \wrt the original infinite-horizon MDP.
    \item Leveraging abstractions as a tool to approximate the optimal value function of the task at hand, we will first provide an experimental back up of our theoretical findings in goal-oriented MDPs, outperforming~\citet{Cipollone2023}. Then, we will show how we can make a significant difference in terms of sample efficiency in the Arcade Learning Environment (ALE) benchmark~\citep{bellemare2013arcade} by accepting the bias induced by PBRS in general MDPs. Our approach achieves performances comparable or superior to the ones of Convolutional Neural Networks (CNNs)~\citep{mnih2015human}, but only using fully-connected architectures. To the best of our knowledge, this is unprecedented.
\end{itemize}

\section{Background}\label{sec:background}

In the RL setting, an agent interacts with the environment to learn a policy. After perceiving the current state of the environment, the agent executes an action that makes the environment transition into a next state. Finally, the agent collects the reward associated with the experienced transition and the cycle starts again. 

\subsection{The Reinforcement Learning Problem}
Problems in the RL paradigm are usually modeled through MDPs. An MDP is described as a tuple $\{\mathcal{S}, \mathcal{A}, \mathcal{P}, r, \gamma, \rho\}$, where $\mathcal{S}$ is the state space, $\mathcal{A}$ is the action space, $\mathcal{P}(s'|s, a)$ is the Markovian state transition function, $r(s, a, s') \in [0,\Bar{R}]$ is the reward function, $\gamma \in [0, 1)$ is the discount factor, and $\rho$ is the initial state distribution. 
Goal-oriented MDPs characterize a sub-family of MDPs $\{\mathcal{S}, \mathcal{S}_G, \mathcal{A}, \mathcal{P}, r, \gamma, \rho\}$, where $\mathcal{S}_G \subset \mathcal{S}$ is a set of goal states, and the reward function $r(s, a, s')=1$ if and only if $s \notin \mathcal{S}_G \wedge s' \in \mathcal{S}_G$; it is $0$ otherwise.

Techniques to solve an RL problem can be divided into Value-based and Policy-based algorithms. The former strive to learn the optimal (action-)value function $V^*(s)$ ($Q^*(s, a)$), whereas the latter aim at directly learning the optimal policy. $Q^*(s, a)$ is defined as the expected obtained return starting in state $s$, executing action $a$, and following the optimal policy thereafter. Instead, $V^*(s)$ can be obtained starting from $Q^*(s, a)$ and selecting $a$ according to the optimal policy, or, equivalently, selecting $a$ greedily \wrt $Q^*(s, a)$. This means that $V^*(s) = \max_a Q^*(s, a)$ and that the optimal policy can be obtained by acting greedily on the environment under the guidance of $Q^*(s, a)$. 

A current standard way to learn the above-mentioned functions is through the use of Deep Q-Networks (DQNs)~\citep{mnih2015human}. In this case,
given a dataset of collected experience $D = \left \langle s_t, a_t, r_t, s_{t+1} \right \rangle_{t=1}^N$, a DQN minimizes the following loss function during learning, leveraging samples of experience drawn uniformly at random from $D$:
\begin{equation*}
    \mathbb{E}_{(s, a, r, s') \sim D}\left[\left(r + \gamma \max_{a'}Q_{\theta^-}(s', a') - Q_{\theta}(s, a)\right)^2 \right],
\end{equation*}
where $Q_\theta$ and $Q_{\theta^-}$ are action-value functions parameterized by the weights $\theta$ and $\theta^-$ that are represented by CNNs~\citep{lecun1998gradient}. They respectively represent the actual action-value function and the target action-value function. The former is always updated while the latter only every $M$ steps.

\subsection{Potential-Based Reward Shaping}\label{sec:PBRS}

The idea behind PBRS~\citep{ng1999policy} is to steer the learning process towards paths that reach optimal behaviour faster by providing additional reward feedback through potential functions. This is accomplished by running the learning algorithm on $\mathcal{M'}=\{\mathcal{S}, \mathcal{A}, \mathcal{P}, r', \gamma, \rho\}$, a different MDP, called the reshaped MDP, where $r'(s, a, s') = r(s, a, s') + F(s, a, s')$ is the only change over the original MDP.
Let us define a potential function as $\phi: \mathcal{S}\rightarrow\mathbb{R}$. If $F(s, a, s') = \gamma \phi(s') - \phi(s)$, $\phi$ is bounded, and the horizon length, $H$, is infinite, we are guaranteed that the optimal policy in $\mathcal{M'}$ is optimal also in $\mathcal{M}$ (this holds also for $\epsilon$-optimal policies)~\citep{ng1999policy}. $\phi$ can be interpreted as a guide, telling the agent how valuable a certain state is. Furthermore, if we use this reshaping approach, the following relationships between (action-)value functions hold: 
\begin{align}
    \mathcal{V}_{\mathcal{M}'}^{\pi} (s) = \mathcal{V}_{\mathcal{M}}^{\pi}(s) - \phi(s) \label{eq:corollary2_ng}\\
    \mathcal{Q}_{\mathcal{M}'}^{\pi} (s, a) = \mathcal{Q}_{\mathcal{M}}^{\pi}(s, a) - \phi(s) \label{eq:corollary2_ng_actionvalue}
\end{align}
as shown by \textbf{Corollary 2} in~\citep{ng1999policy}.
Therefore, if $\phi = \mathcal{V}_{\mathcal{M}}^{\pi^*}$, then the optimal value function in the reshaped MDP is equivalent to zero.

In practice, RL algorithms cannot solve problems using an infinite horizon (as the previous guarantees require). Hence, we usually have to set a finite limit to $H$, $H<\infty$. This generates a bias, since the reshaped return depends on the final state reached by the agent~\citep{grzes2017reward}:
\begin{equation}
    \mathcal{R}'(\tau) = \mathcal{R}(\tau) + \gamma^H\phi(s_H) -\phi(s_0), \label{eq:bias}
\end{equation}
where $\tau = s_0, a_0, s_1, \dots, a_{H-1}, s_H$ is a trajectory and $\mathcal{R}(\tau) = \sum_{t=0}^{H-1}\gamma^t r(s_t,a_t, s_{t+1})$ the discounted return. Since the last visited state depends on the policy, optimizing $\mathcal{R}'(\tau)$ is not the same as optimizing $\mathcal{R}(\tau)$.

\section{Abstractions for PBRS}\label{sec:abstractions4pbrs}
As expected, the potential function choice has an impact on PBRS efficiency.
If we use the optimal value function as the potential function, it represents an excellent choice, because it makes the optimal value function of the reshaped MDP equivalent to zero (an easy function to be learned). But, it is unrealistic to assume we already have the optimal value function. Hence, it is common to use handcrafted heuristics in an attempt to approximate the optimal value function of the original task~\citep{ng1999policy}. 

Inspired by the use of abstractions for heuristics computation in heuristic search and automated planning~\citep{DBLP:journals/ci/CulbersonS98,haslum2007domain},
an alternative approach consists of defining an MDP, ${\mathcal{M}}_\alpha$, that is an abstraction of the original task~\citep{Cipollone2023} formalized by an aggregation function $\alpha$ whose definition can be found later. ${\mathcal{M}}_\alpha$'s optimal value function can then be used as the potential function for PBRS.
This idea works under the intuition that the abstraction should be a simplified version of the original task, \ie with a much smaller state/action space. Thus, generating the optimal value function of ${\mathcal{M}}_\alpha$ will see a strong reduction in computation. Besides, it provides an approximation of the original task's optimal value function. Furthermore, in many tasks, it is much easier for a human to provide knowledge via the definition of a simpler task rather than through the definition of a specific potential function. Take as an example specifying the potential function for a general ALE task that, in its RAM state representation, has $2^{1024}$ states. Defining a potential function over that many states in a meaningful way is a non-trivial endeavor. However, simplifying the game itself through an abstraction that neglects enemies or other complex dynamics of the game is a much easier task that can greatly reduce the number of states/actions to consider.

To formally define reshaping via abstractions, let $\mathcal{S}_\alpha$ represent the state space of ${\mathcal{M}}_\alpha$, with $|\mathcal{S}_\alpha| \ll |\mathcal{S}|$ and $\alpha: \mathcal{S}\rightarrow\mathcal{S}_\alpha$ be an aggregation function, mapping each state of the original task into a state in $\mathcal{S}_\alpha$. Also, let $\mathcal{V}^*_{{\mathcal{M}}_\alpha}(s_\alpha), s_\alpha\in\mathcal{S}_\alpha$ be the optimal value function of ${\mathcal{M}}_\alpha$. Then, we have the following reshaped reward in $\mathcal{M}'$, using $\mathcal{V}^*_{{\mathcal{M}}_\alpha}(s_\alpha)$ as the potential function ($\phi(s) = \mathcal{V}^*_{{\mathcal{M}}_\alpha}(\alpha(s))$):
\begin{equation}
    r'(s, a, s') = r(s, a, s') + \gamma\mathcal{V}^*_{{\mathcal{M}}_\alpha}(\alpha(s')) - \mathcal{V}^*_{{\mathcal{M}}_\alpha}(\alpha(s)). \label{eq:reshaping}
\end{equation}

\section{The Benefit of PBRS}\label{sec:sample_efficiency}
In the previous section, we have given intuition on why the optimal value function constitutes a good choice for the potential and seen how to get an approximation of it. Let us now go deeper into why such a choice should provide a benefit for our algorithms.
To do this, we begin by analyzing the performance impact of PBRS on Value Iteration (VI), which, to the best of our knowledge, has not been considered by related literature.
\begin{restatable}[VI Performance]{proposition}{VI Performance}\label{prop:VI_performance}
Given an MDP $\mathcal{M}$, an initialization point $\mathcal{V}_0$, a potential function $\phi$ that is an approximation of $\mathcal{V}_{\mathcal{M}}^*$ such that $||\mathcal{V}_{\mathcal{M}}^* - \phi - \mathcal{V}_0||_\infty \le ||\mathcal{V}_{\mathcal{M}}^* - \mathcal{V}_0||_\infty$, where $||\cdot||_\infty$ is the infinity norm, applying VI starting from $\mathcal{V}_0$ in both the original and reshaped MDP will result in the minimum number of iterations to guarantee $\epsilon$-optimality in the reshaped MDP to be lower than in the original MDP.
\end{restatable}

\textit{Proposition \ref{prop:VI_performance}} (proof in Appendix~\ref{appx:vi_proof}) tells us that improvements are achievable in the context of Value Iteration whenever we are able to provide a good estimate of the optimal value function through the potential function using PBRS. However, this grants us just computational savings, and we would like to convert it into a greater sample efficiency gain in the context of a general RL setting.

For what concerns value-based algorithms like Q-Learning~\citep{watkins1992q} and SARSA~\citep{rummery1994line},~\citet{wiewiora2003potential} showed that, under the assumption of learning through the same experience, PBRS is equivalent to initializing the action-value function with the potential function. This can be interpreted as a simplistic form of Transfer Learning~\citep{taylor2009transfer, lazaric2012transfer}. Hence, whenever we initialize the action-value function with a potential function that represents a good starting estimate of the optimal solution, we will observe a greater sample efficiency due to a more favorable exploration-exploitation trade-off. On the other hand, in the context of the Policy Gradient Theorem (PGT)~\citep{sutton1999policy}:
\begin{equation}
        \nabla_{\theta}J_{\theta} = \sum_s d^{\pi_{\theta}}(s) \sum_a \pi_{\theta}(a|s)\nabla_{\theta}\log \pi_{\theta}(a|s)\mathcal{Q}_{\mathcal{M}}^{\pi_{\theta}}(s, a)\label{eq:PGT},
\end{equation}
where $d^{\pi_{\theta}}(s) = \sum_{t=0}^{\infty}\gamma^t Pr\{s_t=s|s_0, \pi_{\theta}\}$,
baselines are commonly used to reduce the variance of the gradient estimator. The baseline version of PGT defines $\nabla_{\theta}J_{\theta}$ as:
    \begin{equation*}
        \sum_s d^{\pi_{\theta}}(s) \sum_a \pi_{\theta}(a|s)\nabla_{\theta}\log \pi_{\theta}(a|s)(\mathcal{Q}_{\mathcal{M}}^{\pi_{\theta}}(s, a) - b(s)),
    \end{equation*}
where $b(s)$ is the baseline. A baseline is a free parameter not affecting the estimate of the gradient, indeed $\sum_a \pi_{\theta}(a|s)\nabla_{\theta}\log \pi_{\theta}(a|s)b(s) = 0$, that can reduce its variance~\citep{deisenroth2013survey}.\footnote{$\sum_a \pi_{\theta}(a|s)\nabla_{\theta}\log \pi_{\theta}(a|s)b(s) = 0$ can be obtained via $\nabla_{\theta} \pi_{\theta}(a|s) = \pi_{\theta}(a|s)\nabla_{\theta}\log \pi_{\theta}(a|s)$ and $\sum_a\nabla_{\theta} \pi_{\theta}(a|s)b(s) = b(s)\nabla_{\theta}\sum_a \pi_{\theta}(a|s) = 0$.}

\begin{restatable}[Baseline Equivalence]{proposition}{Baseline Equivalence}\label{prop:Baseline_Equivalence}
In the context of the PGT, using PBRS (\citeauthor{grzes2017reward}' approach in the finite-horizon case) is equivalent to using a baseline $b(s) = \phi(s)$.
\end{restatable}
\begin{proof}
    Applying the PGT to the reshaped MDP $\mathcal{M}'$ via Eq.~\eqref{eq:PGT}:
    \begin{equation*}
            \nabla_{\theta}J_{\theta} = \sum_s d^{\pi_{\theta}}(s) \sum_a \pi_{\theta}(a|s)\nabla_{\theta}\log \pi_{\theta}(a|s)\mathcal{Q}_{\mathcal{M}'}^{\pi_{\theta}}(s, a).
    \end{equation*}
    Now, using Eq.~\eqref{eq:corollary2_ng_actionvalue}, we get $\nabla_{\theta}J_{\theta}$ equivalent to: 
    \begin{equation*}
        \sum_s d^{\pi_{\theta}}(s) \sum_a \pi_{\theta}(a|s)\nabla_{\theta}\log \pi_{\theta}(a|s)(\mathcal{Q}_{\mathcal{M}}^{\pi_{\theta}}(s, a) - \phi(s)).
    \end{equation*}
    If we combine the above expression with $b(s) = \phi(s)$, it yields the baseline version of the PGT. In the case of finite horizons, we are forced to use a non-stationary potential~\citep{devlin2012dynamic} as proposed by~\citet{grzes2017reward} (which is unbiased), so that $\phi(s_H) = 0$ and the Monte-Carlo estimate of $\mathcal{Q}_{\mathcal{M}'}^{\pi_{\theta}}(s, a)$ can be decomposed as above.\footnote{$\phi(s_H) = 0$ only for the states encountered at the end of a trajectory. If the same state is encountered before, its original potential is used.}
\end{proof}

It is widely known that using value function estimates as baselines allows Policy Gradient algorithms to suffer less variance in their gradients~\citep{schulman2015high}. This, in turn, implies a greater sample efficiency because of the more precise direction estimate toward which the algorithm will move during its optimization process. Since \textit{Proposition~\ref{prop:Baseline_Equivalence}} shows the equivalence between baselines and PBRS, we can deduce that a good value function estimate provided by the potential will help us improve sample efficiency, achieving better performance. The main advantage of using PBRS is that it is not mandatory to estimate the value function from data collected by the agent. Instead, the estimate may come from a simplified representation of the problem the agent is tackling. This enables a different way to provide knowledge to the agent while learning and does not require the representation to be perfect, relieving the user of the burden of providing accurate models.

\section{Bias Analysis}\label{sec:bias_analysis}

Now, we focus on the bias induced by PBRS in the episodic setting, \ie whenever we set a finite horizon due to computational limitations. We will start by considering goal-oriented MDPs, and then we will focus on general MDPs to understand the conditions that preserve policy ordering between the original problem and the reshaped one.

\subsection{Goal-Oriented MDPs}
An important family of goal-oriented MDPs is represented by those having a deterministic transition function, hence, we start our analysis from them.

\begin{restatable}[Policy Ordering in Deterministic Goal-Oriented Episodic MDPs]{thrm}{Policy Ordering in Deterministic Goal-Oriented Episodic MDPs}\label{thm:ordering_DGOMDPs_expct}
In the episodic context, let the horizon be $H$, $\mathcal{M}$ be a  deterministic goal-oriented MDP, and $\mathcal{M}'$ be its reshaped version obtained using a bounded potential $\phi: \mathcal{S} \rightarrow [0, \Phi]$. Moreover, let $\phi(s_G)=\Phi~\forall~s_G~\in~\mathcal{S}_{G}$. Then, the policy ordering in $\mathcal{M}$ is preserved in the reshaped MDP $\mathcal{M}'$ for all the deterministic policies reaching the goal in at most $H$ steps. 
\end{restatable}
\begin{proof}
Let $\pi^*$ be the optimal policy in the original MDP (the policy that requires the least number of steps to reach the goal), and $p_{\pi}(\tau) = \rho(s_0)\pi(a_0| s_0)\Pi_{t=1}^{H-1}\mathcal{P}(s_{t+1}|s_t, a_t)\pi(a_t|s_t)~\forall \pi$ be the density of the distribution over the set of all trajectories induced by the policy and the transition dynamics (notice that this is a $\delta$ over one single trajectory because the MDP and the policy are deterministic in the analysed case). Then, to show that the policy ordering is preserved, we want to prove:
    \begin{align}
        &\mathbb{E}_{\tau \sim p_{\pi^*}}\left[\mathcal{R}'(\tau)\right] > \mathbb{E}_{\tau \sim p_{\pi}}\left[\mathcal{R}'(\tau)\right] ~\forall \pi \neq \pi^* \\
        &\mathbb{E}_{\tau \sim p_{\pi^*}}\left[\mathcal{R}(\tau) + \gamma^{h(\tau)}\phi(s_G) - \cancel{\phi(s_0)}\right] >\nonumber \\ 
        &\qquad\mathbb{E}_{\tau \sim p_{\pi}}\left[\mathcal{R}(\tau) + \gamma^{h(\tau)}\phi(s_h^{\pi}) - \cancel{\phi(s_0)}\right] ~\forall \pi \neq \pi^*, \label{eq:reshaped_inequality_expct}
    \end{align}
    where we have used $\mathcal{R}'(\tau) = \mathcal{R}(\tau) + \gamma^H\phi(s_H) -\phi(s_0)$~\citep{grzes2017reward} in \eqref{eq:reshaped_inequality_expct}, and $h(\tau)\leq H$ represents the number of steps from the initial state to the goal in $\tau$ that depends on the policy and environment dynamics. Assuming $\phi(s_G)=\Phi~\forall~s_G~\in~\mathcal{S}_{G}$, then choosing $\pi = \pi^{**}$ (the second-best policy \wrt $\mathcal{M}$) will maximize the right-hand side of the inequality and imply $s_h^{\pi^{**}}=s'_G$, so we can substitute $\phi(s_h^{\pi})=\phi(s_h^{\pi^{**}})=\Phi$:
    \begin{align}
        \mathbb{E}_{\tau \sim p_{\pi^*}}&\left[\mathcal{R}(\tau) + \gamma^{h(\tau)}\Phi \right] > 
        \mathbb{E}_{\tau \sim p_{\pi^{**}}}\left[\mathcal{R}(\tau) + \gamma^{h(\tau)}\Phi \right].  \label{eq:using_goal_oriented_hp_expct}
    \end{align}
    Now, \eqref{eq:using_goal_oriented_hp_expct} is always true because:
    \begin{equation*}
        \mathbb{E}_{\tau \sim p_{\pi^*}}\left[\mathcal{R}(\tau)\right] > \mathbb{E}_{\tau \sim p_{\pi^{**}}}\left[\mathcal{R}(\tau)\right]
    \end{equation*}
    and
    \begin{align}
        &\mathbb{E}_{\tau \sim p_{\pi^{*}}}\left[\gamma\gamma^{h(\tau)-1}\Phi\right] > \mathbb{E}_{\tau \sim p_{\pi^{**}}}\left[\gamma\gamma^{h(\tau)-1}\Phi\right]\nonumber \\
        &\mathbb{E}_{\tau \sim p_{\pi^{*}}}\left[\sum_{t=0}^{H-1}\gamma^t r(s_t,a_t, s_{t+1})\right]>\nonumber\\
        &\qquad\mathbb{E}_{\tau \sim p_{\pi^{**}}}\left[\sum_{t=0}^{H-1}\gamma^t r(s_t,a_t, s_{t+1})\right]\label{eq:using_goal_reward_is_one} \\
        &\mathbb{E}_{\tau \sim p_{\pi^{*}}}\left[\mathcal{R}(\tau)\right] > \mathbb{E}_{\tau \sim p_{\pi^{**}}}\left[\mathcal{R}(\tau)\right],\nonumber
    \end{align}
    where in~\eqref{eq:using_goal_reward_is_one} the reward is one in any goal as per the definition of goal-oriented MDP in Section~\ref{sec:background}.
    
    Furthermore, iterating this process for any policy starting from the best and the second best up to the ones requiring $H-1$ and $H$ steps, we see that the ordering is preserved in the reshaped MDP. For what concerns all the other policies requiring more than $H$ steps to achieve the goal, let us restart from~\eqref{eq:reshaped_inequality_expct} selecting $\pi^H$ instead of $\pi^*$, and $\pi^{H+k}$ for any $k$ instead of $\pi$:
    \begin{align*}
        \mathbb{E}_{\tau \sim p_{\pi^H}}&\left[\mathcal{R}(\tau) + \gamma^H\Phi \right] > \mathbb{E}_{\tau \sim p_{\pi^{H+k}}}\left[\mathcal{R}(\tau) + \gamma^H\phi(s_H) \right].
    \end{align*}
    This inequality is true because $\mathbb{E}_{\tau \sim p_{\pi^H}}\left[\mathcal{R}(\tau)\right]>\mathbb{E}_{\tau \sim p_{\pi^{H+k}}}\left[\mathcal{R}(\tau)\right]=0$ and $\Phi\geq\phi(s_H)$.
\end{proof}

In a deterministic goal-oriented MDP, whenever we fix a horizon length for computational reasons, we are setting a limit on the sensitivity of our agent. Therefore, policies reaching the goal in $H+1$ steps cannot be distinguished from those that reach the goal in more than $H+1$ steps. Bearing this in mind, Theorem~\ref{thm:ordering_DGOMDPs_expct} tells us that, as long as the goal can be reached in at most $H$ steps, the solution to the reshaped problem is equivalent to that of the original problem. 

Even though the family of MDPs considered by Theorem~\ref{thm:ordering_DGOMDPs_expct} has great relevance in practice, it is not able to model problems where there is any sort of environmental stochastic dynamic involved. For this reason, we proceed with our analysis in the context of general goal-oriented MDPs as well. We start by defining the probability of not having reached the goal in $H$ steps for an agent following policy $\pi$:
\begin{equation}
    \Pr_{\tau\sim p_{\pi}}\{s_H\notin S_G\}=\mathbb{E}_{\tau\sim p_{\pi}}\left[\mathbb{I}_{\{s_H\notin S_G\}}\right]\label{eq:prob_game_still_on},
\end{equation}
where $p_{\pi}(\tau) = \rho(s_0)\pi(a_0| s_0)\Pi_{t=1}^{H-1}\mathcal{P}(s_{t+1}|s_t, a_t)\pi(a_t|s_t)$ is the density of the distribution over trajectories induced by the policy and the transition dynamics.
For convenience in notation, we define the event $E = \{s_H\notin S_G\}$. Furthermore, we assume the following:
\begin{restatable}[Minimal Performance Gap]{assumption}{Minimal Performance Gap}\label{ass:min_perf_gap}
Given $\pi_A \in \Pi_H$, where $\Pi_H$ is the set of all policies having a non-zero expected return within $H$ steps in $\mathcal{M}$, and $\pi_B$ a general policy in $\mathcal{M}$, the following performance gap is at least $\epsilon>0$:
    \begin{equation*}
        \inf_{(\pi_A,\pi_B) \in \mathcal{Z}} \mathbb{E}_{\tau\sim p_{\pi_A}}\left[\mathcal{R}(\tau)\right]-\mathbb{E}_{\tau\sim p_{\pi_B}}\left[\mathcal{R}(\tau)\right]>\epsilon,
    \end{equation*}
where $\mathcal{Z} =\{\pi_A,\pi_B:\pi_A\in\Pi_H \wedge\mathbb{E}_{\tau\sim p_{\pi_A}}\left[\mathcal{R}(\tau)\right]>\mathbb{E}_{\tau\sim p_{\pi_B}}\left[\mathcal{R}(\tau)\right]\}$.
\end{restatable}

Now we are ready to state the existence of a potential function that is order preserving for all the policies in the relevant set $\Pi_H$ through the following theorem.

\begin{restatable}[Policy Ordering in Goal-Oriented Episodic MDPs]{thrm}{Policy Ordering in Goal-Oriented Episodic MDPs}\label{thm:ordering_GOMDPs_expct}
In the episodic context, let the horizon be $H$, $\mathcal{M}$ be a goal-oriented MDP, $\mathcal{M}'$ be its reshaped version obtained using a bounded potential $\phi: \mathcal{S} \rightarrow [0, \Phi]$. Moreover, let $\phi(s_G)=\Phi\geq C,  ~\forall~s_G~\in~\mathcal{S}_{G}$ with $C$ being:
\begin{align*}
    \ssup\limits_{(\pi_A,\pi_B) \in \mathcal{Z}}\frac{\gamma^{H-1}\Big(\E\limits_{\tau\sim p_{\pi_B}}\left[\phi(s_H)\mathbb{I}_E\right] - \E\limits_{\tau\sim p_{\pi_A}}\left[\phi(s_H)\mathbb{I}_E\right]\Big)}{\mathbb{E}_{\tau\sim p_{\pi_A}}\left[\mathcal{R}(\tau)\right]-\mathbb{E}_{\tau\sim p_{\pi_B}}\left[\mathcal{R}(\tau)\right]},
\end{align*}
where $\mathcal{Z}$ is defined as in Assumption~\ref{ass:min_perf_gap}.
Then, the policy ordering in $\mathcal{M}$ is preserved in $\mathcal{M}'$ for all the policies in $\Pi_H$.
\end{restatable}

As can be seen from the above theorem (proof in Appendix~\ref{appx:proof_thm_GOMDPs}), in the context of general goal-oriented MDPs, we are not as lucky as with the deterministic case. Nevertheless, we are able to state a condition over the potential's maximum value, $\Phi$, that allows preserving the order for the set of relevant policies $\Pi_H$. This condition involves: the performance gap at the denominator; and two expected values depending on $\Pr\limits_{\tau\sim p_{\pi_B}}E$ and $\Pr\limits_{\tau\sim p_{\pi_A}}E$, which are the first and second terms at the numerator, respectively. Under Assumption~\ref{ass:min_perf_gap}, $C$ is bounded since the smallest denominator we can have is $\epsilon$ and the numerator is the difference between the expected values of a bounded function under two different distributions. Observe that assumptions on the performance gap between policies are common in the performance analysis of RL algorithms~\citep{auer2008near} and that, letting $H$ tend to infinity, $C$ goes to zero recovering the result of \citet{ng1999policy} since $\Phi$ is unconstrained.

The above theorems, requiring the potential function to be maximal in the goal states, avoid the drawback of~\citet{grzes2017reward}\footnote{Notice that, after reaching the goal, the potential can be set to zero to avoid any sort of convergence issue for the returns. Equivalently, in practice, the simulation ends there.}. Indeed, his approach, choosing $\phi(s_H)=0$ in Eq.~\eqref{eq:bias}, leaves all trajectories not ending up in a goal without feedback on how far they terminated \wrt to it, impairing the exploration advantages of PBRS~\citep{Cipollone2023}. Finally, the above theorems have significant practical implications. Theorem~\ref{thm:ordering_DGOMDPs_expct} allows PBRS to be used without concern for the bias due to Eq.~\eqref{eq:bias} in deterministic goal-oriented MDPs. Theorem~\ref{thm:ordering_GOMDPs_expct}, on the other hand, states the existence of a potential function that does not introduce bias in the context of stochastic goal-oriented MDPs, and, even though $C$ in practice cannot be known, in Section~\ref{sec:grid_world}, we will see how even a small value will result in optimal performance for our agent. Together these two theorems may avoid computationally expensive approaches like the off-policy scheme proposed by~\citet{Cipollone2023}. Indeed, their primary focus is on goal-oriented MDPs and they require updating two different agents (effectively doubling computational complexity) to avoid a bias that is not actually there.

\subsection{General MDPs}
In the context of general MDPs, we can establish a condition that preserves policy ordering across the two problems.

\begin{restatable}[Policy Ordering in Reshaped Episodic MDPs]{thrm}{Policy Ordering in Reshaped Episodic MDPs}\label{thm:ordering_REMDPs_expct}
Let $\mathcal{M}$ be an MDP and $\mathcal{M}'$ be its reshaped version obtained using a bounded potential $\phi: \mathcal{S} \rightarrow [0, \Phi]$. Then, there exists a horizon
\begin{equation*}
    H > \log_{\gamma} \inf_{\substack{\pi_A, \pi_B: \\ \mathbb{E}_{\rho}\left[V^{\pi_A}(s)\right] > \mathbb{E}_{\rho}\left[V^{\pi_B}(s)\right]}}\frac{ \mathbb{E}_{\rho}\left[V^{\pi_A}(s)\right] - \mathbb{E}_{\rho}\left[V^{\pi_B}(s)\right] }{\Phi + \frac{\Bar{R}}{1-\gamma}}
\end{equation*}
such that the ordering in the reshaped MDP is preserved \wrt the infinite-horizon version of $\mathcal{M}$.
\end{restatable}

The condition stated in Theorem~\ref{thm:ordering_REMDPs_expct} (proof in Appendix~\ref{appx:proof_thm_REMDPS}) is quite restrictive, because preserving the total ordering of policies is a strong equivalence requirement. Indeed, at the numerator, we have the performance gap between two arbitrary policies in the infinite-horizon version of $\mathcal{M}$, which, even if is greater than zero, may be very small, requiring an incredibly long horizon $H$. Therefore, let us consider the least restrictive version of total ordering, that is, preserving only the optimal policy.

\begin{restatable}[Preserving the Optimal Policy in Reshaped Episodic MDPs]{corollary}{Preserving the Optimal Policy in Reshaped Episodic MDPs}\label{thm:optimum_REMDPs_expct}
Let $\mathcal{M}$ be an MDP and $\mathcal{M}'$ be its reshaped version obtained using a bounded potential $\phi: \mathcal{S} \rightarrow [0, \Phi]$. Then, there exists a horizon
\begin{equation*}
    H > \log_{\gamma} \frac{ \mathbb{E}_{\rho}\left[V^{\pi^*}(s)\right] - \mathbb{E}_{\rho}\left[V^{\pi^{**}}(s)\right] }{\Phi + \frac{\Bar{R}}{1-\gamma}},
\end{equation*}
where $\pi^{**}$ is the second-best policy, such that the optimal policy in the reshaped MDP is equivalent to the one in the infinite-horizon version of $\mathcal{M}$.
\end{restatable}
\begin{proof}
    It is enough to select $\pi_A = \pi^*$ in Theorem \ref{thm:ordering_REMDPs_expct}. This implies that $\pi_B$ must be $\pi^{**}$.
\end{proof}

Now, the horizon length depends on the performance gap between the optimal policy and the second-best policy, a concept often used in the literature~\citep{auer2008near}. This is much less stringent than what stated in Theorem~\ref{thm:ordering_REMDPs_expct}.

\section{Evaluation}\label{sec:eval}

In this section, we experimentally evaluate the impact of abstractions coupled with PBRS on sample efficiency. This will be accomplished in the context of a goal-oriented navigation task and some complex ALE games.

\subsection{Grid World}\label{sec:grid_world}

To corroborate the theoretical findings about goal-oriented MDPs described in the previous section, we evaluate the performance of Q-Learning, Abstraction-based PBRS (A-PBRS) Q-Learning,  and Off-Policy Abstraction-based PBRS (OPA-PBRS) Q-Learning against each other. A-PBRS Q-Learning represents the Q-Learning algorithm with the reshaping described in Equation~\ref{eq:reshaping}. OPA-PBRS Q-Learning is the off-policy approach proposed by~\citet{Cipollone2023}. This comparison was done over the $8$-rooms goal-oriented navigation task~\citep{Cipollone2023} depicted in Figure~\ref{fig:grid}. The agent has to move from the top-left room (red) to the right-most one (green) in order to accomplish the task. The abstraction involves considering only the rooms and their connections rather than the connections between the cells (see Appendix~\ref{appx:grid_world}). Further experimental details may be found in Appendix~\ref{appx:exp}. 

This task has a stochastic transition function, so it falls under the scope of Theorem~\ref{thm:ordering_GOMDPs_expct}. Additionally, we cannot compute the theoretical value of $\Phi$. However, even setting $\Phi=1$ allows us not only to achieve the optimum, but also to beat OPA-PBRS Q-Learning as can be seen in Figure~\ref{fig:comparison_degiacomo}. Q-Learning performance is reported for completeness.
\input{grid}
\subsection{Arcade Learning Environment}
\input{atari}
We also show how the Abstraction-based PBRS approach behaves in much more challenging tasks, ALE games~\citep{bellemare2013arcade}. We have chosen three tasks of increasing level of difficulty and significantly different among each other in terms of dynamics: \textit{Freeway}, \textit{Q*Bert}, and \textit{Venture}. Each abstraction for these problems has been developed by simplifying the game neglecting enemies or entire areas/dynamics of the game itself. In the following section, we provide a  complete and detailed description of \textit{Freeway} and its abstraction, whereas we defer the reader to Appendix~\ref{appx:qbert} and~\ref{appx:venture} for \textit{Q*Bert}, \textit{Venture}, and their respective abstractions due to space constraints.

To tackle the above-mentioned tasks, we will use DQN (an overview of the algorithm can be found in Section~\ref{sec:background}). In the experiments, we will compare DQN based on fully-connected architectures against the same algorithm, but this time applied on a reshaped environment, where the reshaping is based on Equation~\ref{eq:reshaping}. This is done to evaluate the contribution of the abstraction on sample efficiency. Due to the complexity of the tackled tasks, we will try to adhere as much as possible to theory. However, deviations from it will be allowed whenever these make the implementation feasible or they yield better performance (implementation details in Section~\ref{appx:freeway} and Appendix~\ref{appx:qbert} and~\ref{appx:venture}). The previously mentioned algorithms, using fully-connected neural networks as value function estimators, have been trained on the RAM version of the above mentioned ALE games. This implies that the state space of those environments will be represented by a $128$-dimensional vector whose components may take up to $256$ values encoding various information about the current state of the game (\eg coordinates of the agent, coordinates of the enemies, etc.). To further assess the contribution over sample efficiency and performance achievable through the abstractions, we will additionally compare the fully-connected version of DQN against its CNN version first introduced by~\citet{mnih2015human}. DQN based on CNN architectures will be trained straight from images (the more classical representation of the state space in ALE).

\subsubsection{Abstracting Freeway}\label{appx:freeway}
In \textit{Freeway} (Figure~\ref{fig:freeway_game} in the Appendix), the agent needs to cross a road avoiding the cars driving along the $10$ different lanes. Every time the agent bumps into a car, it will get thrown backward a few steps, whereas upon crossing the entire road it will get one point and restart from the beginning. The game lasts $2$ minutes and $16$ seconds and the overall objective is to cross the entire road as many times as possible.

\textit{Freeway} represents the simplest task we tackled in the context of the ALE benchmark because fully-connected neural networks are able to achieve more than $85\%$ of the CNNs performance on their own. In this task, we propose a very simple abstraction where we neglect the presence of the cars and we model a simple empty road to be crossed by the agent. The agent starts at position $6$ (only the $y$ coordinate is considered because the agent cannot move horizontally) and has to cross position $180$ by moving up, down, or doing nothing ($1$,$2$, and $0$, respectively). Movement actions will increment or decrement the position of the agent by $1$ in the respective direction. Since there are no cars in this abstraction, it is pointless to maximize the amount of times the agent crosses the road; it is sufficient to consider just one crossing. This abstraction has $177$ states and its aggregation function is the following:
\begin{align*}
    &\alpha_{s}(s) = (s_{14}, 0)\\
    &\alpha_{s'}(s, s') = (s'_{14}, s'_{103}-s_{103}),
\end{align*}
where we have used the $14^{th}$ and $103^{rd}$ coordinates ($y$ and score, respectively~\citep{anand2019unsupervised}) of the RAM vector representing the state space of the original task.\footnote{$\alpha_s$ aggregates $s$, whereas $\alpha_{s'}$ aggregates $s'$.} This aggregation function allows us to deal with a smaller abstraction because we do not have to consider multiple crossings of the whole road, but only one. Finally, this proposed abstraction is solved through VI and its solution is used to perform reshaping on the original task as per Equation~\ref{eq:reshaping}.

\subsubsection{Results}\label{sec:results}
In Figures~\ref{fig:freeway},~\ref{fig:qbert}, and~\ref{fig:venture}, we report the learning curves associated to the above described algorithms in the context of \textit{Freeway}, \textit{Q*Bert}, and \textit{Venture}, respectively. DQN with fully-connected neural networks and without reshaping is termed DQN, its equivalent version leveraging Abstraction-based PBRS is called DQN-Reshaped, and DQN based on CNN is called CNN-DQN.

In the context of \textit{Freeway}, as shown in Figure~\ref{fig:freeway}, we observe higher sample efficiency of DQN-Reshaped \wrt DQN even though it is not able to achieve the same final performance as CNN-DQN. This fact is due to the usage of a very simplistic abstraction stemming from the assumption of no cars on the road (an exhaustive description of the abstraction may be found in Section~\ref{appx:freeway}). Unfortunately, accounting for the cars could not be done in the context of this work because of the complex dynamic triggered whenever the agent bumps into a car that could not be reverse engineered from interactions nor from the assembly code of the game.

In \textit{Q*Bert}, Figure~\ref{fig:qbert}, flat fully-connected neural networks, when coupled with PBRS and abstractions (exhaustive description in Appendix~\ref{appx:qbert}), are able to achieve the same performances of CNNs with approximately one fourth the amount of interactions. Additionally, DQN-Reshaped manages to have the best average performance at the end of the learning process. 

In the context of \textit{Venture} (see Figure~\ref{fig:venture}), not only we have greater sample efficiency, but reshaping and abstractions (exhaustive description in Appendix~\ref{appx:venture}) allow fully-connected architectures to achieve the best performance by a great margin \wrt  all the other solutions. 

Finally, the quality of the abstraction, in terms of how well it approximates the original task, plays an important role. Indeed, we can observe a significant difference between the improvement in performance achieved in the context of \textit{Freeway} \wrt \textit{Q*Bert} or \textit{Venture}. This is due to the different amount of knowledge provided by a simplistic abstraction (in \textit{Freeway}) compared to a more complex one (in \textit{Q*Bert} and \textit{Venture}). Furthermore, the longer the sequence of actions to be executed in order to experience reward, the higher the benefit of reshaping. As a matter of fact, experiencing reward in \textit{Freeway} and \textit{Q*Bert} is quite easy, whereas, in \textit{Venture}, the agent is required to execute a long plan in order to get any positive feedback from the environment. This is reflected by the fact that abstractions in \textit{Venture} provide the best improvement. All the experimental details may be found in Appendix~\ref{appx:exp}.\footnote{Code will be made available upon request.}

\section{Related Works}\label{sec:related_works}

\citet{ng1999policy} were the first authors to formalize theoretical properties of reward shaping showing that PBRS preserves optimal and near-optimal policies.~\citet{wiewiora2003potential} proves that, under the same set of collected experience, for Q-learning, SARSA, and other TD-based algorithms, a suitable action-value function initialization is equivalent to PBRS. The first work formally discussing the bias introduced by reshaping when dealing with finite horizons can be found in~\citep{eck2016potential}. In the context of POMDPs, they show that the evaluation of policies becomes the same as the horizon length tends to infinity. Furthermore, they provide a condition under which two policies, one optimizing the reshaped POMDP and one the original, will be different from each other. This is done to understand when the optimal policy in the finite-horizon reshaped POMDP differs from the optimal policy of the original finite-horizon POMDP. In contrast, our Theorem~\ref{thm:ordering_REMDPs_expct} produces a condition to control the bias \wrt to the infinite-horizon MDP, the true problem of interest. Besides, Theorem~\ref{thm:ordering_DGOMDPs_expct} and~\ref{thm:ordering_GOMDPs_expct} establish that, performing reshaping when the horizon is finite does not change the policy ordering \wrt the starting finite-horizon problem for the goal-oriented case (under some assumptions if the MDP is stochastic).~\citet{grzes2017reward}, instead, tackle the above-mentioned bias by setting the potential function at the end of the trajectory to zero, hence leveraging non-stationary potential functions~\citep{devlin2012dynamic}, whereas~\citet{forbes2024potential} using terms that do not depend on the future actions of the agent.

Along another direction,~\citet{cheng2021heuristic} and~\citet{burden2021latent} learn their potential. The former learns it beforehand with offline data using Monte-Carlo Regression, whereas the latter through direct interactions with the environment. Instead, to get the potential, we do not need data or the original environment to interact with while training.

There is not much theoretical analysis in the literature making explicit the performance improvements attainable with reward shaping. One exception is the work of~\citet{gupta2022unpacking}, where they provide a regret expression for a new version of UCBVI called UCBVI-Shaped. However, they assume to have a value function estimator upper-bounding the optimal value function when multiplied by a parameter $\beta$ and their approach is meant for finite state and action problems. Fully self-supervised reward-shaping techniques~\citep{zheng2018learning, memarian2021self} may not be able to deal with sparse reward settings because they lack an explicit exploration incentive. Therefore, to mitigate this limitation,~\citet{devidzeexploration} propose to learn an intrinsic reward function combined with a count-based exploration bonus aiming at accelerating the agent's endeavors in maximizing the extrinsic reward. Similarly,~\citet{wang2024efficient} propose an intrinsic reward to encourage exploration. This is achieved defining the discrepancy between states as differences of potentials, where the potential function is a learned inverse dynamic bi-simulation metric. Along a different direction,~\citet{hu2020learning} study how to learn the best way to utilize a reshaping function in order to solve the RL problem at hand. They do this through a bi-level optimization approach, where at the first level they optimize the policy to maximize the reshaped rewards, and at the second level they optimize the reshaping function for actual reward maximization. Differently \wrt these works we study the bias introduced by PBRS and explicitly leverage abstractions to choose potentials.

The work presented in this paper also draws inspiration from abstractions and how they are used in classical planning~\citep{knoblock1994automatically}.
The idea is to automatically build a simplified version of the problem before the search starts.
This is obtained by abstracting away all but a (small) part of the task, resulting in a new problem that can be optimally solved fast~\citep{haslum2007domain}. 
The plans that solve these abstract problems are then used as heuristics that help guiding the search algorithm toward the goal in the actual problem. Closely related to the above-mentioned way of using abstractions is the concept of abstractions for MDPs~\citep{congeduti2022cross, starre2023analysis}. However, they usually deal with finite state and action MDPs. Furthermore, they directly employ the solutions obtained from the abstractions onto the real problem like~\citet{wang2022causal, wang2024building}.

Finally,~\citet{Cipollone2023} propose an off-policy RL scheme to leverage hierarchies of abstractions and improve sample efficiency. They mainly focus on goal-oriented MDPs for which they provide a theoretical analysis on the exploration quality of the behavioral policy in their proposed solution that depends on the size of the abstraction state space. Hence, with sufficiently large abstractions (\eg on the order of $10^5$ states) easily achievable in the context of ALE games (see Section~\ref{sec:eval}), the bound is not helpful. Furthermore, thanks to Theorem~\ref{thm:ordering_DGOMDPs_expct} and~\ref{thm:ordering_GOMDPs_expct}, we can avoid using their off-policy scheme, which requires twice the computational complexity to maintain convergence to the optimum. In the context of general MDPs, their approach still represents a sound way to preserve convergence guarantees. However, whenever we cannot afford the computational burden of training two separate agents, thanks to Theorem~\ref{thm:ordering_REMDPs_expct}, we can try to control the amount of bias we introduce in the reshaped MDP.

\section{Conclusions and Future Work}\label{sec:conclusions}
In this paper, we have tackled the problem of sample efficiency in RL through abstraction and PBRS, an idea that was recently introduced in the literature by~\citet{Cipollone2023}. Contrary to them, we have embraced the bias of PBRS in the context of finite horizons, investigated its effects on policy ordering for goal-oriented and general MDPs, and assessed the contribution to sample efficiency of abstractions coupled with PBRS in highly complex domains such as the ALE benchmark. Given the impressive sample efficiency achieved, devising ways to learn abstractions automatically from interactions becomes an important avenue for future developments. This could be accomplished by taking inspiration from the Hierarchical Reinforcement Learning literature~\citep{klissarov2023deep}. Moreover, being able to build abstractions over state spaces represented by images is another interesting future direction for this work that could be pursued through learning an action schema directly from images~\citep{asai2018classical}. Finally, the idea of using abstractions as a guide to solve the original task is deeply rooted into the planning literature and the concept of heuristics; it would be worth investigating their connections~\citep{canonaco2024projection}.
\section*{Disclaimer}
This paper was prepared for informational purposes by the Artificial Intelligence Research group of JPMorgan Chase \& Co. and its affiliates ("JP Morgan'') and is not a product of the Research Department of JP Morgan. JP Morgan makes no representation and warranty whatsoever and disclaims all liability, for the completeness, accuracy or reliability of the information contained herein. This document is not intended as investment research or investment advice, or a recommendation, offer or solicitation for the purchase or sale of any security, financial instrument, financial product or service, or to be used in any way for evaluating the merits of participating in any transaction, and shall not constitute a solicitation under any jurisdiction or to any person, if such solicitation under such jurisdiction or to such person would be unlawful.

© 2024 JPMorgan Chase \& Co. All rights reserved.

\bibliography{bibliography}

\newpage

\onecolumn

\appendix

\section{Proofs}\label{appx:proofs}

\subsection{Proof of Proposition~\ref{prop:VI_performance}}\label{appx:vi_proof}
\begin{proof}
    Applying VI in the original MDP $\mathcal{M}$, where $\mathcal{V}_{n, \mathcal{M}}^{src}$ is the output after $n$ iterations, yields the following bound by the $\gamma$-contraction property:
	\begin{equation}
		||\mathcal{V}_{\mathcal{M}}^* - \mathcal{V}_{n, \mathcal{M}}^{src}||_{\infty}\le \gamma^n||\mathcal{V}_{\mathcal{M}}^* - \mathcal{V}_0||_\infty. \label{eq:VI_bound_original}
	\end{equation}
Applying VI in the reshaped MDP $\mathcal{M'}$, where $\mathcal{V}_{n', \mathcal{M}}^{dst}$ is the output after $n'$ iterations brought back into the original MDP through Eq.~\eqref{eq:corollary2_ng}, yields the following bound by the $\gamma$-contraction property:
    \begin{equation}
        ||\mathcal{V}_{\mathcal{M}}^* - \cancel{\phi} - \mathcal{V}_{n', \mathcal{M}}^{dst} + \cancel{\phi}||_{\infty} \le \gamma^{n'}||\mathcal{V}_{\mathcal{M}}^* - \phi - \mathcal{V}_0||_\infty. \label{eq:VI_bound_reshaped}
    \end{equation}
Now, imposing inequalities \eqref{eq:VI_bound_original} and \eqref{eq:VI_bound_reshaped} both smaller than $\epsilon$, we get: 
\begin{align*}
    &||\mathcal{V}_{\mathcal{M}}^* - \mathcal{V}_{n, \mathcal{M}}^{src}||_{\infty} \le \gamma^n||\mathcal{V}_{\mathcal{M}}^* - \mathcal{V}_0||_\infty \le \epsilon 
    \\
    &||\mathcal{V}_{\mathcal{M}}^* - \mathcal{V}_{n', \mathcal{M}}^{dst}||_{\infty} \le \gamma^{n'}||\mathcal{V}_{\mathcal{M}}^* - \phi - \mathcal{V}_0||_\infty \le \epsilon.
\end{align*}
Using the fact that $\gamma < 1$,  we can rewrite the inequalities above as:
\begin{align*}
    &n \geq \log_{\gamma}\bigg(\frac{\epsilon}{||\mathcal{V}_{\mathcal{M}}^* - \mathcal{V}_0||_\infty}\bigg)
    \\
    &n' \geq \log_{\gamma}\bigg(\frac{\epsilon}{||\mathcal{V}_{\mathcal{M}}^* - \phi - \mathcal{V}_0||_\infty}\bigg)
\end{align*}
then, since we assumed $||\mathcal{V}_{\mathcal{M}}^* - \phi - \mathcal{V}_0||_\infty \le ||\mathcal{V}_{\mathcal{M}}^* - \mathcal{V}_0||_\infty$, which can be made true through $\phi$ being an approximation of $\mathcal{V}_{\mathcal{M}}^*$, we have:
\begin{align*}
    \log_{\gamma}\bigg(\frac{\epsilon}{||\mathcal{V}_{\mathcal{M}}^* - \phi - \mathcal{V}_0||_\infty}\bigg) \leq \log_{\gamma}\bigg(\frac{\epsilon}{||\mathcal{V}_{\mathcal{M}}^* - \mathcal{V}_0||_\infty}\bigg)
\end{align*}
hence, VI applied on the reshaped problem will require a smaller minimum number of iterations to guarantee an $\epsilon$-optimal solution. Notice that, the superscripts \textit{src} and \textit{dst} are used to disambiguate the two solutions of VI that may be different.
\end{proof}

\subsection{Proof of Theorem~\ref{thm:ordering_GOMDPs_expct}}\label{appx:proof_thm_GOMDPs}
\begin{proof}
    Let $\pi^*$ be the optimal policy in the original MDP, and $p_{\pi}(\tau) = \rho(s_0)\pi(a_0| s_0)\Pi_{t=1}^{H-1}\mathcal{P}(s_{t+1}|s_t, a_t)\pi(a_t|s_t)~\forall \pi$ be the density of the distribution over the set of all trajectories induced by the policy and the transition dynamics. 

To show that the policy ordering is preserved, we want to prove:
    \begin{align}
        &\mathbb{E}_{\tau\sim p_{\pi^*}}\left[\mathcal{R}'(\tau)\right]>\mathbb{E}_{\tau \sim p_{\pi}}\left[\mathcal{R}'(\tau)\right]~\forall\pi\neq\pi^*,\pi\in\Pi_H\\
        &\mathbb{E}_{\tau\sim p_{\pi^*}}\left[\mathcal{R}(\tau)+\gamma^{h(\tau)}\phi(s_h^{\pi^*})-\cancel{\phi(s_0)}\right]>\nonumber\\
        &\qquad\mathbb{E}_{\tau \sim p_{\pi}}\left[\mathcal{R}(\tau)+\gamma^{h(\tau)}\phi(s_h^{\pi})-\cancel{\phi(s_0)}\right]~\forall\pi\neq\pi^*, \pi\in\Pi_H,\label{eq:reshaped_inequality_expct1}
    \end{align}
    where we have used $\mathcal{R}'(\tau) = \mathcal{R}(\tau) + \gamma^H\phi(s_H) -\phi(s_0)$~\citep{grzes2017reward} in \eqref{eq:reshaped_inequality_expct1}, and $h(\tau)\leq H$ represents the random variable associated to the number of steps from the initial state to the goal in $\tau$ that depends on the policy and environment dynamics.
    We know that $\mathbb{E}_{\tau\sim p_{\pi^*}}\left[\mathcal{R}(\tau)\right]>\mathbb{E}_{\tau \sim p_{\pi}}\left[\mathcal{R}(\tau)\right]~\forall\pi\neq\pi^*$, so we only have to prove:
    \begin{align*}
        \mathbb{E}_{\tau \sim p_{\pi^*}}\left[\gamma^{h(\tau)}\phi(s_h^{\pi^*})\right]>\mathbb{E}_{\tau\sim p_{\pi}}\left[\gamma^{h(\tau)}\phi(s_h^{\pi})\right]~\forall\pi\neq\pi^*,\pi\in\Pi_H.
    \end{align*}
    Then:
    \begin{align*}
        \mathbb{E}_{\tau\sim p_{\pi^*}}\left[\gamma^{h(\tau)}\phi(s_h^{\pi^*})\mathbb{I}\{s_h^{\pi^*}\in S_G\}\right]+\mathbb{E}_{\tau\sim p_{\pi^*}}\left[\gamma^{H}\phi(s_H)\mathbb{I}\{s_H\notin S_G\}\right]>\\\mathbb{E}_{\tau\sim p_{\pi}}\left[\gamma^{h(\tau)}\phi(s_h^{\pi})\mathbb{I}\{s_h^{\pi}\in S_G\}\right]+\mathbb{E}_{\tau\sim p_{\pi}}\left[\gamma^{H}\phi(s_H)\mathbb{I}\{s_H\notin S_G\}\right],
    \end{align*}
    where $\mathbb{I}\{s_h^{\pi}\in S_G\}$ is the indicator function selecting only the trajectories that end in a goal state in at most $H$ steps, and $\mathbb{I}\{s_H\notin S_G\} = \mathbb{I}_E$ ($E$ defined in Section~\ref{sec:bias_analysis}) selects all the trajectories not ending in the goal within $H$ steps. Now, since $\mathbb{E}_{\tau\sim p_{\pi}}\left[\gamma^{h(\tau)}\phi(s_h^{\pi})\mathbb{I}\{s_h^{\pi}\in S_G\}\right]=\gamma\mathbb{E}_{\tau\sim p_{\pi}}\left[\mathcal{R}(\tau)\right]\Phi~\forall\pi$ because in goal-oriented MDPs $\mathbb{E}_{\tau\sim p_{\pi}}\left[\gamma^{h(\tau)-1}\mathbb{I}\{s_h^{\pi}\in S_G\}\right]=\mathbb{E}_{\tau\sim p_{\pi}}\left[\gamma^{h(\tau)-1}r(s_{h(\tau)-1},a_{h(\tau)-1},s_{h(\tau)})\right]$, then:
    \begin{equation*}
        \gamma\left(\mathbb{E}_{\tau\sim p_{\pi^*}}\left[\mathcal{R}(\tau)\right]-\mathbb{E}_{\tau\sim p_{\pi}}\left[\mathcal{R}(\tau)\right]\right)\Phi> 
        \gamma\left(\mathbb{E}_{\tau\sim p_{\pi}}\left[\gamma^{H-1}\phi(s_H)\mathbb{I}_E\right]-\mathbb{E}_{\tau\sim p_{\pi^*}}\left[\gamma^{H-1}\phi(s_H)\mathbb{I}_E\right]\right)
    \end{equation*}
    \begin{align*}
    \Phi>\frac{\gamma^{H-1}\left(\mathbb{E}_{\tau\sim p_{\pi}}\left[\phi(s_H)\mathbb{I}_E\right]-\mathbb{E}_{\tau\sim p_{\pi^*}}\left[\phi(s_H)\mathbb{I}_E\right]\right)}{\mathbb{E}_{\tau\sim p_{\pi^*}}\left[\mathcal{R}(\tau)\right]-\mathbb{E}_{\tau\sim p_{\pi}}\left[\mathcal{R}(\tau)\right]}.
\end{align*}
The above is true because we chose $\Phi\geq C$. We can repeat all the above reasoning for $\pi^{**}$ instead of $\pi^*$ such that $\pi\notin\{\pi^*,\pi^{**}\}$, and so on up until we exhaust all the policies in $\Pi_H$.

Now taking $\pi_i\in\Pi_H$ and $\pi_o\notin\Pi_H$, we have to prove:
\begin{align*}
    \mathbb{E}_{\tau\sim p_{\pi_i}}\left[\mathcal{R}(\tau)+\gamma^{h(\tau)}\phi(s_h^{\pi_i})\right]>\mathbb{E}_{\tau\sim p_{\pi_o}}\left[\gamma^H\phi(s_H)\right]~\forall\pi_o\notin\Pi_H
\end{align*}
that holds true if:
\begin{align*}
    \Phi>\frac{\gamma^{H-1}\left(\mathbb{E}_{\tau\sim p_{\pi_o}}\left[\phi(s_H)\mathbb{I}_E\right]-\mathbb{E}_{\tau\sim p_{\pi_i}}\left[\phi(s_H)\mathbb{I}_E\right]\right)}{\mathbb{E}_{\tau\sim p_{\pi_i}}\left[\mathcal{R}(\tau)\right]}
\end{align*}
that is smaller than the chosen $C$ for any $\pi_i$, $\pi_o$.
\end{proof}

\subsection{Proof of Theorem~\ref{thm:ordering_REMDPs_expct}}\label{appx:proof_thm_REMDPS}

\begin{proof}
Given the definition of $p_{\pi}(\tau)$ (see Section~\ref{sec:bias_analysis}), let $p_{\pi}^\infty(\tau)$ be its limit for $H$ going to infinity. Furthermore, let $\mathbb{E}_{p_{\pi_A}}\left[\mathcal{R}'(\tau_A)\right]$ and $\mathbb{E}_{p_{\pi_B}}\left[\mathcal{R}'(\tau_B)\right]$ be the performance of two generic policies in the reshaped MDP such that $\mathbb{E}_{\rho}\left[V^{\pi_A}(s)\right]>\mathbb{E}_{\rho}\left[V^{\pi_B}(s)\right]$ in the original MDP. Then, for the ordering to be preserved in the reshaped MDP, the following must hold:
    \begin{align}
        &\mathbb{E}_{p_{\pi_A}}\left[\mathcal{R}'(\tau_A)\right] > \mathbb{E}_{p_{\pi_B}}\left[\mathcal{R}'(\tau_B)\right] ~\forall \pi_A, \pi_B \nonumber \\
        &\mathbb{E}_{p_{\pi_A}}\left[\mathcal{R}(\tau_A) + \gamma^H\phi(s_H) - \cancel{\phi(s_0)}\right] > \mathbb{E}_{p_{\pi_B}}\left[\mathcal{R}(\tau_B) + \gamma^H\phi(s_H) - \cancel{\phi(s_0)}\right] ~\forall \pi_A, \pi_B \label{eq:ordering_preserving_cond} \\
        &\mathbb{E}_{p_{\pi_A}}\left[\mathcal{R}(\tau_A)\right] - \mathbb{E}_{p_{\pi_B}}\left[\mathcal{R}(\tau_B)\right] > \mathbb{E}_{p_{\pi_B}}\left[\gamma^H\phi(s_H)\right] - \mathbb{E}_{p_{\pi_A}}\left[\gamma^H\phi(s_H)\right] ~\forall \pi_A, \pi_B, \nonumber
    \end{align}
    where in~\eqref{eq:ordering_preserving_cond} we have used: $\mathcal{R}'(\tau) = \mathcal{R}(\tau) + \gamma^H\phi(s_H) -\phi(s_0)$.
    The potential is bounded, $\phi \in [0,\Phi]$. Hence, given $\mathcal{R}(\tau^\infty) = \sum_{t=H}^{\infty}\gamma^t r(s_t,a_t, s_{t+1})$, we have:
    \begin{align}
        &\mathbb{E}_{p_{\pi_A}}\left[\mathcal{R}(\tau_A)\right] - \mathbb{E}_{p_{\pi_B}}\left[\mathcal{R}(\tau_B)\right] >  \gamma^H \Phi ~\forall \pi_A, \pi_B \label{eq:bned_potential} \\
        &\mathbb{E}_{p_{\pi_A}}\left[\mathcal{R}(\tau_A)\right] - \mathbb{E}_{p_{\pi_B}}\left[\mathcal{R}(\tau_B)\right] + \mathbb{E}_{p_{\pi_A}^\infty}\left[ \mathcal{R}(\tau_A^\infty)\right] - \mathbb{E}_{p_{\pi_B}^\infty}\left[ \mathcal{R}(\tau_B^\infty)\right] > \gamma^H\Phi + \mathbb{E}_{p_{\pi_A}^\infty}\left[ \mathcal{R}(\tau_A^\infty)\right] - \mathbb{E}_{p_{\pi_B}^\infty}\left[ \mathcal{R}(\tau_B^\infty)\right] \nonumber\\
        &\mathbb{E}_{\rho}\left[V^{\pi_A}(s)\right] - \mathbb{E}_{\rho}\left[V^{\pi_B}(s)\right]> \gamma^H\Phi + \mathbb{E}_{p_{\pi_A}^\infty}\left[ \mathcal{R}(\tau_A^\infty)\right]\label{eq:bned_reward0}\\
        &\mathbb{E}_{\rho}\left[V^{\pi_A}(s)\right] - \mathbb{E}_{\rho}\left[V^{\pi_B}(s)\right]> \gamma^H\Phi + \frac{\gamma^H}{1-\gamma}\Bar{R}\label{eq:bned_reward1}\\
        &\gamma^H < \inf_{\substack{\pi_A, \pi_B: \\ \mathbb{E}_{\rho}\left[V^{\pi_A}(s)\right] > \mathbb{E}_{\rho}\left[V^{\pi_B}(s)\right]}}\frac{ \mathbb{E}_{\rho}\left[V^{\pi_A}(s)\right] - \mathbb{E}_{\rho}\left[V^{\pi_B}(s)\right] }{\Phi + \frac{\Bar{R}}{1-\gamma}} \nonumber \\
        &H > \log_{\gamma} \inf_{\substack{\pi_A, \pi_B: \\ \mathbb{E}_{\rho}\left[V^{\pi_A}(s)\right] > \mathbb{E}_{\rho}\left[V^{\pi_B}(s)\right]}}\frac{ \mathbb{E}_{\rho}\left[V^{\pi_A}(s)\right] - \mathbb{E}_{\rho}\left[V^{\pi_B}(s)\right] }{\Phi + \frac{\Bar{R}}{1-\gamma}}, \nonumber
    \end{align}
where, in~\eqref{eq:bned_potential}, we have used the fact that the potential is bounded within $[0, \Phi]$ to upper bound the right-hand side. Instead, in~\eqref{eq:bned_reward0} and~\eqref{eq:bned_reward1}, we have used the fact that the reward function is bounded within $[0, \Bar{R}]$ to upper bound the right-hand side.
\end{proof}

\section{Abstractions}\label{appx:abstractions}

\subsection{Grid World}\label{appx:grid_world}
In this section, we describe the 8-rooms abstraction, reported as a graph in Figure~\ref{fig:grid_world_abstraction}. Each node is associated with a room with the topology of the graph and the coloring faithfully representing Figure~\ref{fig:grid} (this also defines the aggregation function). The connections between nodes are bidirectional and represent the available actions. Additionally, any action has a $90\%$ chance of success and a $10\%$ failure probability. This means that, with probability $0.9$, the agent ends up in the intended target node, with probability $0.1$, instead, it will stay where it is. Closely following~\citet{Cipollone2023}, this stochastic behavior is implemented in the original 8-rooms environment of Figure~\ref{fig:grid} as well, but, in this case, the agent has a $4\%$ failure probability upon executing any action, and, when it fails, one among all the available actions is executed at random. Finally, from the practical point of view, it is worth pointing out that, in order for the potential to be maximum in the goal state as required by Theorem~\ref{thm:ordering_GOMDPs_expct}, we add in the abstraction a fictitious done action that gives us a reward of 1 and transitions us from $\alpha(s_G)$ into a fictitious terminal state of the abstraction itself (this also simplifies the implementation when dealing with multiple goals in the abstraction because all have to transition to the same terminal state).

\begin{figure}[!h]
    \centering
    \begin{tikzpicture}

      \node[circle, draw, fill=red!30, minimum size=1cm] (S) at (0,3) {S};
      \node[circle, draw, fill=violet!30, minimum size=1cm] (b) at (3,3) {$R_1$};
      \node[circle, draw, fill=yellow!30, minimum size=1cm] (r) at (0,0) {$R_2$};
      \node[circle, draw, fill=blue!30, minimum size=1cm] (y) at (3,0) {$R_3$};
      \node[circle, draw, fill=pink!30, minimum size=1cm] (p) at (5,1.5) {$R_4$};
      \node[circle, draw, fill=brown!30, minimum size=1cm] (o) at (7,3) {$R_5$};
      \node[circle, draw, fill=orange!30, minimum size=1cm] (B) at (7,0) {$R_6$};
      \node[circle, draw, fill=green!30, minimum size=1cm] (G) at (9,1.5) {G};
    
      \draw (S) -- (b);
      \draw (S) -- (r);
      \draw (r) -- (y);
      \draw (y) -- (p);
      \draw (p) -- (o);
      \draw (p) -- (B);
      \draw (o) -- (B);
      \draw (B) -- (G);
    
    \end{tikzpicture}
    \vspace{0.15in}
    \caption{Grid World abstraction.}
    \label{fig:grid_world_abstraction}
    \vspace{0.25in}
\end{figure}

\subsection{Q*Bert}\label{appx:qbert}
\begin{figure}
    \centering
    \begin{subfigure}[b]{0.15\textwidth}
        \includegraphics[scale=0.5]{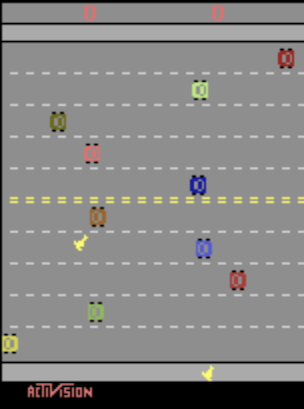}
        \caption{Freeway.}
        \label{fig:freeway_game}
    \end{subfigure}
    \begin{subfigure}[b]{0.15\textwidth}
        \includegraphics[scale=0.5]{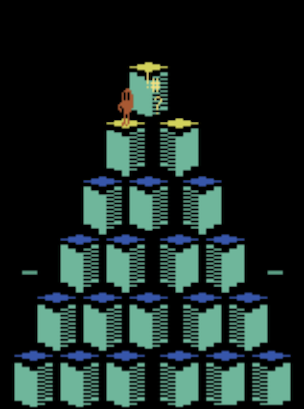}
        \caption{Q*Bert.}
        \label{fig:qbert_game}
    \end{subfigure}
    \begin{subfigure}[b]{0.15\textwidth}
        \includegraphics[scale=0.5]{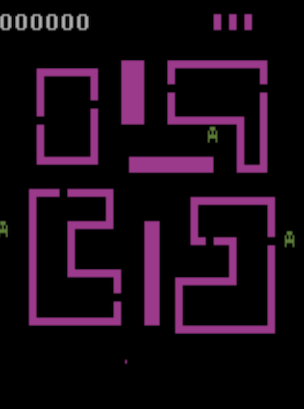}
        \caption{Venture.}
        \label{fig:venture_game}
    \end{subfigure}
    \vspace{0.15in}
    \caption{The tackled ALE games.}
    \vspace{0.2in}
\end{figure}
In \textit{Q*Bert} (Figure~\ref{fig:qbert_game}), the agent can jump onto some tiles over a grid. By jumping onto tiles, the agent is able to change their color. The goal of the agent is to change every tile on the grid to a target color of interest. This will make it succeed and go to the next level. The agent is awarded $25$ points for each tile turned to the target color. Furthermore, it is awarded a bonus of $3100$ points on level completion. Unfortunately, while the agent attempts to accomplish its task, there are enemies attempting at its life or undoing its tile-coloring work.

\subsubsection{Abstracting Q*Bert}
In order to abstract away what the agent has to accomplish in \textit{Q*Bert}, we neglect all the enemies and consider only the positions of the tiles on the grid. In a nutshell, the agent can move between tiles, changing their color upon visiting them. There are $21$ positions the agent can be in, four movement actions (left, right, up, and down), and one action that simulates the death of the agent, bringing it back to the initial position without undoing the coloring that has been done so far. The goal of the agent is to visit every tile in order to color it with the target color. The reward in this abstraction will be $0$ for every step except upon succeeding with the task, where the reward is $1$. This abstraction has $1,172,830$ states. We solved it using VI, and the $\epsilon$-optimal value function is used to guide the agent through reshaping in the original task.

\subsubsection{The Abstraction Details}

The nodes in Figure~\ref{fig:qbert_transition_graph} represent the tiles the agent has to color. The agent starts at tile $1$ that has coordinates $(77, 25)$ (see Table~\ref{table:id2position}) and it can move according to the connections reported in the graph of Figure~\ref{fig:qbert_transition_graph}. For instance, if the agent is at node $5$, it can move to $3$ by executing the \textit{up} action, move to $2$ by executing \textit{left}, move to $9$ by executing \textit{right}, move to $8$ by executing \textit{down}, and, finally, \textit{die} there by moving back directly to $1$. To avoid cluttering, the dying action has been reported only for node $6$ in Figure~\ref{fig:qbert_transition_graph}; it applies to all the nodes, and it does not change the color of the landing tile (which is always $1$). All the other movement actions will change the color of the landing tile, and they have been reported only for node $5$ for the sake of readability (the pattern is the same for every node except for the borders where missing arcs mean that the associated actions are not available for execution). The state space of the abstraction is then represented by the $x$ and $y$ coordinates of the agent corresponding to the coordinates reported in Table~\ref{table:id2position} and $21$ variables representing the colors of the tiles that start all at the same source color and need to be transitioned all to the target by the agent. Formally, being $s$ the state vector of the original task, we have:
\begin{align*}
    \Bar{s} &= \left(s_{43}, s_{67},  s_{21}, s_{52}, s_{54}, s_{83}, s_{85}, s_{87}, s_{98}, s_{100}, s_{102}, s_{104}, s_{1}, s_{3}, s_{5}, s_{7}, s_{9}, s_{32}, s_{34}, s_{36}, s_{38}, s_{40}, s_{42}\right)\\
    &= (\underbrace{x, y,}_{\text{Coordinates}} \underbrace{1, 2, 3, 4, 5, 6, 7, 8, 9, 10, 11, 12, 13, 14, 15, 16, 17, 18, 19, 20, 21}_{\text{Node IDs reported in Figure~\ref{fig:qbert_transition_graph}}})
\end{align*}

\begin{figure}[!th]
    \centering
    \begin{tikzpicture}[node distance=0.8cm]
      \foreach \i/\y/\z in {1/0/0, 3/-1/1, 2/-1/-1, 6/-2/2, 5/-2/0, 4/-2/-2, 10/-3/3, 9/-3/1, 8/-3/-1, 7/-3/-3, 15/-4/4, 14/-4/2, 13/-4/0, 12/-4/-2, 11/-4/-4, 21/-5/5, 20/-5/3, 19/-5/1, 18/-5/-1, 17/-5/-3, 16/-5/-5} {
        \node[circle, draw, minimum size=0.5cm] (\i) at (\z, \y) {\i};
      }
    
      \draw (1) -- (2);
      \draw (1) -- (3);
      
      \draw (2) -- (4);
      \draw (2) -- (5);
      \draw (3) -- (5);
      \draw (3) -- (6);
    
      \draw (4) -- (2);
      \draw (4) -- (7);
      \draw (4) -- (8);
      \draw (5) -- (2) node[midway] {left};
      \draw (5) -- (3) node[midway] {up};
      \draw (5) -- (8) node[midway] {down};
      \draw (5) -- (9) node[midway] {right};
      \draw (6) -- (3);
      \draw (6) -- (9);
      \draw (6) -- (10);
      \draw[->] (6) to[out=45, in=45] node[pos=0.5, above] {death} (1);
    
      \draw (7) -- (4);
      \draw (7) -- (11);
      \draw (7) -- (12);
      \draw (8) -- (4);
      \draw (8) -- (5);
      \draw (8) -- (12);
      \draw (8) -- (13);
      \draw (9) -- (5);
      \draw (9) -- (6);
      \draw (9) -- (13);
      \draw (9) -- (14);
      \draw (10) -- (6);
      \draw (10) -- (14);
      \draw (10) -- (15);
    
      \draw (11) -- (7);
      \draw (11) -- (16);
      \draw (11) -- (17);
      \draw (12) -- (7);
      \draw (12) -- (8);
      \draw (12) -- (17);
      \draw (12) -- (18);
      \draw (13) -- (8);
      \draw (13) -- (9);
      \draw (13) -- (18);
      \draw (13) -- (19);
      \draw (14) -- (9);
      \draw (14) -- (10);
      \draw (14) -- (19);
      \draw (14) -- (20);
      \draw (15) -- (10);
      \draw (15) -- (20);
      \draw (15) -- (21);
    \end{tikzpicture}
    \vspace{0.15in}
    \caption{Q*Bert Transition Graph. The Node ID correspond to the position in $x$, $y$ coordinates reported in Table~\ref{table:id2position}.}
    \label{fig:qbert_transition_graph}
    \vspace{0.3in}
\end{figure}

\begin{table*}[!th]
    \caption{Node ID to position mapping.}\label{table:id2position}
    \vspace{0.1in}
	\centering
	\begin{tabular}{| l | c | l | c|}
		\hline	
		Node ID & $(x, y)$ & Node ID & $(x, y)$ \\
        \hline
		1 & (77, 25) & 12& (53, 137)\\
		2 & (65, 53) & 13& (77, 137)\\
        3 & (93, 53) & 14& (105, 137)\\
        4 & (53, 81) & 15& (129, 137)\\
        5 & (77, 81) & 16& (16, 165)\\
        6 & (105, 81) & 17& (41, 165)\\
        7 & (41, 109) & 18& (65, 165)\\
        8 & (65, 109) & 19& (93, 165)\\
        9 & (93, 109) & 20& (117, 165)\\
        10 & (117, 109) & 21& (141, 165)\\
        \cline{3-4}
        11 & (29, 137) &\multicolumn{2}{|c|}{}\\
        \hline
	\end{tabular}
 
\end{table*}

\subsubsection{The Coloring Across Different Levels}

In the context of the original task, on every level, we have a different source-target color pair that can be detected automatically at the first few interactions of the agent with the new level. This is done by monitoring the transitions the agent is experiencing. Whenever we go to a new level ($57^{th}$ component of the RAM vector), we grab the color of tile $1$, then as soon as we transition from $1$ to either $2$ or $3$, and the color of one of those two tiles has changed \wrt the one we grabbed in $1$, we save the source-target color pair. This automatic mechanism allows us to make our abstraction agnostic to the color change across different levels, so that the abstraction is kept fixed throughout the whole learning process and does not depend on the level the agent is at. Finally, notice that, handling all possible different source-target color pairs not only enables us to have a single abstraction for all the levels, but it also enables a pre-processing step that translates all the source-target color pairs into a single one to help the network better leverage similarities across levels. In Figure~\ref{fig:qbert}, we have reported results using abstraction and pre-processing combined, whereas in Figure~\ref{fig:qbert_ablation}, we report an ablation study to better assess the contribution of the single components: only abstraction and only pre-processing. As we may see, with reshaping, we are able to greatly improve performances \wrt the basic DQN solution. The same may be stated for what concerns leveraging the pre-processing at the cost of a greater variance. The two techniques combined yield the best performance surpassing even CNN-based DQN as stated in Section~\ref{sec:results}.
\input{ablation}
\begin{algorithm}[!h]
	\caption{Q*Bert F function}
	\label{alg:qbert_F}
	\begin{algorithmic}[1]
            \STATE {\bfseries Input:} $\left<\Bar{s}, a, \Bar{s}'\right>$ the transition, $\mathcal{G}$ the abstraction's transition graph, $\phi = \mathcal{V}_\epsilon$ the potential function
        \IF {$\Bar{s}_{x,y}$ not in $\mathcal{G}$}
            \STATE \textbf{return} $0$
        \ENDIF
        \IF{$\Bar{s}_{x,y}==\Bar{s}'_{x,y}$}
            \STATE \textbf{return} $0$
        \ENDIF
        \IF{$a$ not in $\mathcal{G}\left[s\right]$}
            \IF {$\Bar{s}$ in $\phi$}
                \STATE \textbf{return} $-\phi(\Bar{s})$
            \ELSE
                \STATE \textbf{return} $0$
            \ENDIF
        \ELSIF{$\Bar{s}$ in $\phi$ and $\text{\textbf{next\_state}}(\Bar{s}, \mathcal{G}\left[\Bar{s}_{x,y}\right]\left[a\right])$ in $\phi$}
            \STATE $F = \gamma\phi(\text{\textbf{next\_state}}(\Bar{s}, \mathcal{G}\left[\Bar{s}_{x,y}\right]\left[a\right]))-\phi(\Bar{s})$
            \IF{F==0}
                \STATE \textbf{return} $0.1$
            \ENDIF
            \STATE \textbf{return} $F$
        \ELSE
            \STATE \textbf{return} $0$
        \ENDIF
	\end{algorithmic}
\end{algorithm}
\subsubsection{The Aggregation Function}

In order to complete the map between $\mathcal{S}$ and $\mathcal{\Bar{S}}$, we need to account for the transition dynamics between the nodes connected via the actions \textit{up}, \textit{right}, \textit{down}, and \textit{left} in Figure~\ref{fig:qbert_transition_graph}.\footnote{Note that this mapping induces a partition over the original state space where each $\Bar{s}$ is a class representative.} In our abstraction that transition is instantaneous, but in the real task it is not. Let us see an example:

\begin{equation*}
    \underbrace{(77, 25)}_{\Bar{s}^1_{x,y}}, \underbrace{(78, 24), (82, 20), (86, 22), (89, 27), (89, 35), (89, 43), (89, 51), (93, 53)}_{\Bar{s}^2_{x,y}}.
\end{equation*}
This represents the transitions undergone by the $x$ and $y$ coordinates of the agent when moving from tile $1$ to tile $3$ in the original task. This is the same transition, but at level $2$:
\begin{equation*}
    \underbrace{(77, 25)}_{\Bar{s}^1_{x,y}}, \underbrace{(81, 21), (85, 21), (89, 25), (89, 33), (89, 41), (89, 49), (93, 53)}_{\Bar{s}^2_{x,y}}.
\end{equation*}
Furthermore, to make things even harder, they keep on changing throughout the different levels. Since this transition is mandatory, and the agent cannot deviate from it, we are going to define the aggregation function taking the endpoints as representatives and assigning the in-between states as shown above. Formally, let $\sigma_1, \dots, \sigma_k = \sigma \in \varsigma$ be a possible sequence of succeeding $x$ and $y$ coordinates allowed by the game between $(77, 25)$ and $(93, 53)$, then the aggregation function for this family of transitions will map $(77, 25) \rightarrow (77, 25)$ and $\{\sigma, (93, 53)\} \rightarrow (93, 53)~\forall \sigma \in \varsigma $, and so on for every other node in Figure~\ref{fig:qbert_transition_graph}. In order to induce the above described partitioning, we are going to use the action chosen by the agent. Indeed, given the agent is at $(77, 25)$, if it executes the action \textit{right}, it will observe the appropriate $\sigma \in \varsigma$ and then eventually $(93, 53)$. This implies that we are able to properly reshape between $(77, 25)$ and the first $x$ and $y$ coordinate pair, $\sigma_1$, in the $\sigma$ sequence. The reshaping of all the other transitions, namely, $\left<\sigma_2, \sigma_3\right>,\dots,\left<\sigma_k, (93, 53)\right>$, is set to $0$.

The pseudo-code that shows an implementation of the $F$ function (see Section~\ref{sec:PBRS}) for \textit{Q*Bert} is shown in Algorithm~\ref{alg:qbert_F}, where we are assuming that $\mathcal{G}$ is an hash-map representing the transition graph reported in Figure~\ref{fig:qbert_transition_graph} without the death arcs. $\mathcal{V}_{\epsilon}$ is the solution produced by VI applied to the abstraction. Finally, the function \textbf{next\_state}, given the current abstract state and the $x$ and $y$ coordinates the agents ends up at after executing $a$, returns the abstract next state by appropriately coloring the required tile in $\Bar{s}$. In Line $1$, we do not reshape if the starting state coordinates are out of our abstraction. In Line $5$, we do not reshape if the agent stays still; this is done not to incentivize continuous movement. In Lines $8$ to $13$, if the agent moved out of the abstraction from a known state, we penalize it with $-\phi(s)$. This is because it is jumping out of the grid and consequently dying. Otherwise, if the state is unknown, we do not reshape (Line $12$). In Line $14$, we handle the transition between one abstract tile and another. If the value of $F$ is $0$ we boost it to $0.1$ (this happens every time we take an action that is optimal according to $\mathcal{V}_\epsilon$), otherwise, we simply return $F$.\footnote{Notice that our abstraction is a goal-oriented MDP solved with the same discount factor used in the original task, see Section~\ref{appx:exp}.} In all other remaining cases $0$ is returned.

Changing the reshaping function as described above deviates from theory and it may introduce bias, but as we have seen in the experiments of Figure~\ref{fig:qbert}, the achieved performance is better than CNN-based DQN. Additionally, it is worth noticing that, in Algorithm~\ref{alg:qbert_F}, for the sake of readability, we did not report the pseudo-code to handle all the possible source-target color pairs. Finally, we normalize the reward of the original task dividing it by $500$ (the highest reward obtainable in a single step). This does not change the task from the optimization point of view and make everything on a similar scale \wrt reshaping values.

\subsection{Venture}\label{appx:venture}
In \textit{Venture} (Figure~\ref{fig:venture_game}), the agent is placed into a main hall that must be explored in order to find some rooms where treasures are hidden. In order to progress, the agent has to enter each room, collect its treasure, and exit the room without being killed by the enemies and pitfalls present in every room (main hall included). The agent will receive $200$ points for each collected treasure and $100$ points for each killed enemy in the treasure rooms (these last points are awarded only if the killing happens after the treasure has been collected). This reward system makes the task a sparse reward problem, which is much harder to deal with for RL agents.

\subsubsection{Abstracting Venture}
In the context of \textit{Venture}, the objective of the agent is twofold: exploration to find the rooms' doors in the main hall and treasure collection in the rooms. Whenever the agent enters a room, collects its treasure, and exits without being killed, the room gets locked and can no longer be accessed. Therefore, we can define a simple abstraction modeling only the problem of finding the rooms' doors. This can be done by considering the position of the agent, whether a room is locked or not, the room the agent is currently in, and the boundaries of the main hall. Again, as we did in \textit{Q*Bert}, we will neglect the enemies that chase the agent in the main hall. In the context of the above-described abstraction, visiting a room via one of its doors makes that room transition into the previously-mentioned locked state. Hence, we are abstracting away what the agent must do inside the rooms. The available actions are the movements along the eight main cardinal directions, which will change the components of the agent's positional vector in the appropriate way (increment/decrement by one unit or no change). The goal is to visit all the rooms. When this happens, the agent receives a reward of $1$, $0$ everywhere else. This abstraction has $106,929$ states and is again solved $\epsilon$-optimally with VI. The obtained value function is used to do reshaping as stated in Equation~\ref{eq:reshaping}. The aggregation function is:
\begin{equation*}
    \alpha(s) = (s_{85}, s_{26}, s_{90}, s_{17}),
\end{equation*}
which means we are considering only the $85^{th}$, $26^{th}$, $90^{th}$, and $17^{th}$ components ($x$ and $y$ coordinates, room id, and locked rooms, respectively \citep{anand2019unsupervised}) of the original RAM vector.

In the same spirit as the abstraction above, we could build a second layer of abstraction for the rooms within the main hall. Unfortunately, neglecting the monsters does not allow us to build a good abstraction due to their aggressive behaviour. Indeed, reshaping in these rooms without accounting for the monsters would provide misleading feedback to the agent because the shortest path to the treasure and back out of the room heavily depends on the enemies' positions. Thus, we lean on the agent to find a way to kill the monsters, collect the treasure, and exit. 

\subsubsection{The Abstraction Details}
\begin{figure}[!t]
    \centering
    \begin{tikzpicture}[scale=0.1]
      \draw[thick] (1,78) rectangle (160,0);
    
      \filldraw[fill=gray!30] (16,44) -- (16,8) -- (65,8) -- (65,24) -- (41,24) -- (41,32) -- (61,32) -- (61,44) -- cycle;
      \node at (16,44) [below right, font=\tiny] {(16,34)};
      \node at (16,8) [above right, font=\tiny] {(16,70)};
      \node at (65,8) [above left, font=\tiny] {(65,70)};
      \node at (65,24) [below left, font=\tiny] {(65,54)};
      \node at (41,24) [below, font=\tiny] {(41,54)};
      \node at (41,32) [above, font=\tiny] {(41,46)};
      \node at (61,32) [above left, font=\tiny] {(61,46)};
      \node at (61,44) [below left, font=\tiny] {(61,34)};
      \node at (65,16) [right, font=\tiny, text=red] {(65,62)};
      \node at (36,44) [above, font=\tiny, text=red] {(36,34)};
      \node at (30,29) {Room 1};
    
      \filldraw[fill=gray!30] (100,42) -- (145,42) -- (145,6) -- (92,6) -- (92,22) -- (120,22) -- (120,28) -- (100,28) -- cycle;
      \node at (100,42) [below right, font=\tiny] {(100,36)};
      \node at (145,42) [below left, font=\tiny] {(145,36)};
      \node at (145,6) [above left, font=\tiny] {(145,72)};
      \node at (92,6) [above right, font=\tiny] {(92,72)};
      \node at (92,22) [below right, font=\tiny] {(92,56)};
      \node at (120,22) [below left, font=\tiny] {(120,56)};
      \node at (120,28) [above left, font=\tiny] {(120,50)};
      \node at (100,28) [above right, font=\tiny] {(100,50)};
      \node at (112,28) [below, font=\tiny, text=red] {(112,50)};
      \node at (145,30) [right, font=\tiny, text=red] {(145,48)};
      \node at (131,27) {Room 2};
    
      \filldraw[fill=gray!30] (88,76) -- (141,76) -- (141,46) -- (124,46) -- (124,58) -- (88,58) -- cycle;
      \node at (88,76) [below right, font=\tiny] {(88,2)};
      \node at (141,76) [below left, font=\tiny] {(141,2)};
      \node at (141,46) [above left, font=\tiny] {(141,32)};
      \node at (124,46) [above right, font=\tiny] {(124,32)};
      \node at (124,58) [above right, font=\tiny] {(124,20)};
      \node at (88,58) [above right, font=\tiny] {(88,20)};
      \node at (88,69) [left, font=\tiny, text=red] {(88,9)};
      \node at (141,68) [right, font=\tiny, text=red] {(141,10)};
      \node at (118,67) {Room 3};
    
      \filldraw[fill=gray!30] (20,74) -- (54,74) -- (54,48) -- (20,48) -- cycle;
      \node at (20,74) [below right, font=\tiny] {(20,4)};
      \node at (54,74) [below left, font=\tiny] {(54,4)};
      \node at (54,48) [above left, font=\tiny] {(54,30)};
      \node at (20,48) [above right, font=\tiny] {(20,30)};
      \node at (54,64) [right, font=\tiny, text=red] {(54,14)};
      \node at (20,60) [left, font=\tiny, text=red] {(20,18)};
      \node at (37,61) {Room 4};
    
      \filldraw[fill=gray!30] (64,76) -- (79,76) -- (79,58) -- (64,58) -- cycle;
      \node at (64,76) [left, font=\tiny] {(64,2)};
      \node at (79,76) [right, font=\tiny] {(79,2)};
      \node at (79,58) [right, font=\tiny] {(79,20)};
      \node at (64,58) [left, font=\tiny] {(64,20)};
    
      \filldraw[fill=gray!30] (68,52) -- (114,52) -- (114,46) -- (68,46) -- cycle;
      \node at (68,52) [above, font=\tiny] {(68,26)};
      \node at (114,52) [above, font=\tiny] {(114,26)};
      \node at (114,46) [below, font=\tiny] {(114,32)};
      \node at (68,46) [below, font=\tiny] {(68,32)};
    
      \filldraw[fill=gray!30] (76,36) -- (85,36) -- (85,8) -- (76,8) -- cycle;
      \node at (76,36) [above, font=\tiny] {(76,42)};
      \node at (85,36) [above, font=\tiny] {(85,42)};
      \node at (85,8) [below, font=\tiny] {(85,70)};
      \node at (76,8) [below, font=\tiny] {(76,70)};
    \end{tikzpicture}
    \vspace{0.15in}
    \caption{Venture abstraction: main hall.}
    \label{fig:venture_main_hall}
    \vspace{0.15in}
\end{figure}
In Figure~\ref{fig:venture_main_hall}, we can find a complete description of the main-hall abstraction with all the rooms, the obstacles, and, in red, the doors to access the rooms. The agent starts in position $(68, 76)$. Furthermore, notice that whenever we are over a door position, by playing the appropriate action, we will enter the room and end up over the corresponding door position on the other side. For instance, if the agent is at $(65,62)$, by playing the action \textit{left}, it will end up being at $(129, 63)$ in Room 1 (see Table~\ref{table:venture_door2door}). From there, the agent will only be able to go back by playing \textit{right} effectively locking the room, as if it was solved. In Algorithm~\ref{alg:venture_F}, we report the pseudo-code describing an implementation of the reshaping mechanism in the context of \textit{Venture}. Here, the only potential sources of bias are due to Line $3$ of Algorithm~\ref{alg:venture_F}, where we return 0 if the agent stays still to avoid incentives to continuous movement, and in Algorithm~\ref{alg:venture_F_8}, where we return $0$ if one of the two abstract states is out of the abstraction. This was done because during preliminary experimental studies it appeared to give better performances.

\begin{table*}[!t]\caption{Main hall door to room door position mapping}\label{table:venture_door2door}
    \vspace{0.1in}
	\centering
	\begin{tabular}{| c | c | c |}
		\hline	
		Main Hall Door & Action & Room Door \\
        \hline
		(65, 62) & \textit{left}& (129, 63)\\
		(36, 34) & \textit{bottom}& (58, 11)\\
        (112, 50) & \textit{up}& (62, 18)\\
        (145, 48) & \textit{left}& (129, 13)\\
        (141, 10) & \textit{left}& (129, 15)\\
        (88, 9) & \textit{right}& (31, 15)\\
        (54, 14) & \textit{left}& (117, 39)\\
        (20, 18) & \textit{right}& (43, 39)\\
        \hline
	\end{tabular}
 
\end{table*}

\begin{algorithm}[!t]
	\caption{Venture F function}
	\label{alg:venture_F}
	\begin{algorithmic}[1]
        \STATE {\bfseries Input:} $\left<s, a, s'\right>$ the transition, $\phi^0 = \mathcal{V}^0_\epsilon$ the potential function associated with the main hall
        \IF {$s_{85,26} == s'_{85,26}$}
            \STATE \textbf{return} $0$
        \ENDIF
        \IF {$s_{90} == 8$}
            \STATE \textbf{return $F_8$} $(s, a, s', \phi^0)$ 
        \ENDIF
        \STATE \textbf{return} $0$
	\end{algorithmic}
\end{algorithm}

\begin{algorithm}[!t]
	\caption{Venture $F_8$ function}
	\label{alg:venture_F_8}
	\begin{algorithmic}[1]
        \STATE {\bfseries Input:} $\left<s, a, s'\right>$ the transition, $\phi = \mathcal{V}^0_\epsilon$ the potential function associated to the main hall
        \STATE $\Bar{s} = s_{85, 26, 90, 17}$, $\Bar{s}' = s'_{85, 26, 90, 17}$
        \IF {$\Bar{s}$ in $\phi$ and $\Bar{s}'$ in $\phi$}
            \STATE $F = \gamma\phi(\Bar{s}') -\phi(\Bar{s})$
            \STATE \textbf{return $10^3F$} 
        \ENDIF
        \STATE \textbf{return} $0$
	\end{algorithmic}
\end{algorithm}

\section{Experimental Details}\label{appx:exp}
\subsection{Value Iteration}
In this section, we will describe all the VI configurations to solve the above-described abstractions. For all the abstractions, we have an $\epsilon$ of $10^{-7}$, whereas the discount factor $\gamma$ is $0.98$ for \textit{Freeway} and \textit{Venture}, $0.99$ for \textit{Q*Bert}, and $0.9$ for the Grid World.
\subsection{Deep Q-Networks}

Regarding the ALE benchmark, we used Mushroom-RL~\citep{JMLR:v22:18-056} implementation of DQN with the hyperparameters reported in Table~\ref{table:hyperparams}. Notice that, for \textit{Freeway}, the \textbf{test exploration rate} has been lowered to $0.01$ because it yielded better performance in evaluation for the CNN architecture. Furthermore, being \textit{Freeway} a very simple task to solve, we lowered the \textbf{evaluation frequency} to $100000$ and the \textbf{test samples} to $50000$. This enables more frequent probing of the agent's performance, allowing a higher resolution for the learning curve, which, in turn, allows us to better visualize and understand the learning dynamics, especially in the early epochs where it matters the most. 

Adam~\citep{KingBa15} is used as optimizer with a learning rate of $0.0001$ and $(\beta_1, \beta_2) = (0.9, 0.999)$. In the case of RAM states, we have fully-connected neural networks with the following structure: one input layer of $128\cdot4$ neurons, two hidden layers of $128$ and $256$ neurons (first and second hidden layers, respectively), both with ReLU activation functions, and an output layer whose size depends on the number of actions (\textit{Freeway} has $3$, \textit{Q*Bert} $6$, and \textit{Venture} $18$). In the case of image states, we have an input shape of $(4, 84, 84)$, and the CNN structure is as follows: $3$ convolutional layers and $2$ linear layers. The first convolutional layer has $4$ input channels, $32$ output channels, a kernel size of $8$, and stride of $4$; the second one has $32$ input channels, $64$ output channels, a kernel size of $4$, and stride of $2$; the third one has $64$ input channels, $64$ output channels, a kernel size of $3$, and stride of $1$. The first linear layer has $3136$ neurons and the second $512$. The output is the number of actions, which depends on the task and it is already stated above. Every layer has a ReLU activation function except the output layer that does not have any.

\begin{table*}[!h]\caption{DQN Hyperparameters}\label{table:hyperparams}
	\centering
    \vspace{0.1in}
	\begin{tabular}{| l | c | r |}
		\hline	
		Parameter & Value & Description\\
        \hline
        $\gamma$ & 0.99 & The discount factor \\
		initial replay size & 50000& Initial number of samples in the replay memory\\
		max replay size & 500000& Replay memory max capacity\\
        batch size & 32 & Number of samples for each fit of the network\\
        history Length & 4 & Number of frames a state is made of\\
        target update frequency & 10000 & Number of collected samples before a new update of the target network\\
        train frequency & 4 & Number of collected samples before a new fit of the neural network\\
        final exploration frame & 1000000& The number of samples while the exploration rate decays\\
        initial exploration rate & 1 & The initial exploration rate value\\
        final exploration rate & 0.1 & The final exploration rate value \\
        max no-op actions & 30 & Maximum amount of no-ops executed at the beginning of each episode \\
        max steps & 50000000 & Total number of interactions\\
        evaluation frequency & 250000 & Total number of interactions before evaluating the policy\\
        test samples & 125000 & Interactions during policy evaluation\\
        test exploration rate & 0.05 & Exploration rate used while evaluating the policy\\
        \hline
	\end{tabular}
 
\end{table*}

\subsection{Q-Learning}
In this section, we report all the hyperparameters used to train Q-Learning in the context of the 8-rooms task. The learning rate starts at $0.85$ for vanilla Q-Learning, $0.7$ for our solution, and $1$ for \citet{Cipollone2023} solution. The different starting points have been optimized for each algorithm independently using a grid search over the following values: $[1, 0.85, 0.7, 0.55, 0.4, 0.25, 0.1]$. They all decay to $0.01$. The parameter $\epsilon$ of the $\epsilon$-greedy policy starts at $1$ and decays to $0.1$. Both the learning rate and the exploration parameter decay over $300000$ time steps (the whole learning process). Finally, we have an episode maximum length of $70$ and a discount factor of $0.98$.

\end{document}

%% file: grid.tex
\begin{figure}[!ht]
    \centering
    \begin{tikzpicture}
        [
        box/.style={rectangle, draw=gray, minimum size=0.3cm},
        dot/.style = {circle, minimum size=0.2cm,
              inner sep=0pt, outer sep=0pt},
        ]
        \foreach \x in {0,0.3,...,7.5}{
            \foreach \y in {0,0.3,...,3.3}
                \node[box] at (\x,\y){};
        }
        \foreach \x in {0,0.3,...,7.5}{
            \node[box, fill=gray] at (\x,0) {};
        }
        \foreach \x in {0,0.3,...,7.5}{
            \node[box, fill=gray] at (\x, 3) {};
        }
        \foreach \y in {0,0.3,...,3.3}{
            \node[box, fill=gray] at (0, \y) {};
        }
        \foreach \y in {0,0.3,...,3.3}{
            \node[box, fill=gray] at (7.2, \y) {};
        }
        \foreach \y in {0,0.3,...,3.3}{
            \node[box, fill=gray] at (1.5, \y) {};
        }
        \foreach \y in {0,0.3,...,3.3}{
            \node[box, fill=gray] at (3, \y) {};
        }
        \foreach \y in {0,0.3,...,3.3}{
            \node[box, fill=gray] at (4.5, \y) {};
        }
        \foreach \y in {0,0.3,...,3.3}{
            \node[box, fill=gray] at (6.3, \y) {};
        }
        \foreach \x in {0,0.3,...,7.5}{
            \node[box, fill=gray] at (\x, 1.5) {};
        }
        \foreach \x in {3.3,3.6,...,4.5}{
            \node[box, fill=gray] at (\x, 0.3) {};
        }
        \foreach \x in {3.3,3.6,...,4.5}{
            \node[box, fill=gray] at (\x, 2.7) {};
        }
        \foreach \x in {6.6,6.9,...,7.2}{
            \node[box, fill=gray] at (\x, 0.6) {};
        }
        \foreach \x in {0.3,0.6,...,1.5}{
            \foreach \y in {1.8,2.1,...,3.0}
                \node[box, fill=red!30] at (\x,\y){};
        }
        \foreach \x in {0.3,0.6,...,1.5}{
            \foreach \y in {0.3,0.6,...,1.5}
                \node[box, fill=yellow!30] at (\x,\y){};
        }
        \foreach \x in {1.8,2.1,...,3.0}{
            \foreach \y in {0.3,0.6,...,1.5}
                \node[box, fill=blue!30] at (\x,\y){};
        }
        \foreach \x in {1.8,2.1,...,3.0}{
            \foreach \y in {1.8,2.1,...,3.0}
                \node[box, fill=violet!30] at (\x,\y){};
        }
        \foreach \x in {3.3,3.6,...,4.5}{
            \foreach \y in {0.6,0.9,...,2.4}
                \node[box, fill=pink!30] at (\x,\y){};
        }
        \foreach \x in {4.8,5.1,...,6.3}{
            \foreach \y in {0.3,0.6,...,1.5}
                \node[box, fill=orange!30] at (\x,\y){};
        }
        \foreach \x in {4.8,5.1,...,6.3}{
            \foreach \y in {1.8,2.1,...,3.0}
                \node[box, fill=brown!30] at (\x,\y){};
        }
        \foreach \x in {6.6,6.9}{
            \foreach \y in {0.9,1.2}
                \node[box, fill=green!30] at (\x,\y){};
        }
        \node[box, fill=orange!30] at (6.3, 0.9) {};
        \node[box, fill=yellow!30] at (0.9, 1.5) {};
        \node[box, fill=yellow!30] at (1.5, 0.9) {};
        \node[box, fill=blue!30] at (3, 0.6) {};
        \node[box, fill=red!30] at (1.5, 2.4) {};
        \node[box, fill=pink!30] at (4.5, 2.4) {};
        \node[box, fill=orange!30] at (4.5, 0.9) {};
        \node[box, fill=brown!30] at (4.8, 1.5) {};
        \node[dot, fill=red] at (0.6,2.1) {};

        % ...
    \end{tikzpicture}
    \vspace{0.15in}
    \caption{8-rooms task. The red dot is the agent, the goal is in green, and the gray blocks represent walls. Each color is a macro-state exhaustively defining the aggregation function.}
    \label{fig:grid}
    \vspace{0.2in}
\end{figure}

%% file: atari.tex
\begin{figure*}[!t]
	\centering
	%
	%
	\begin{tikzpicture}
	\begin{customlegend}[legend columns=8,legend style={align=left,draw=none,column sep=2ex,font=\footnotesize},legend entries={ DQN, CNN-DQN, DQN-Reshaped}]
	%
	%
	%
	\addlegendimage{green1!60!black,ultra thick,dashed,every mark/.append style={solid}, mark=triangle*}   
	\addlegendimage{orange!80!white,dotted,ultra thick,every mark/.append style={solid}, mark=square*}
	\addlegendimage{blue1,ultra thick,every mark/.append style={solid},mark=*}
	%
	
	\end{customlegend}
	\end{tikzpicture}
\\
\begin{subfigure}{0.45\textwidth}
	%
	\begin{tikzpicture}
	\begin{axis}[
	width=\textwidth,
	height=5cm,
	xmin=0,
	xmax=300000,
	xtick={0,50000,...,300000},
	ymin=20,
	ymax=80,
	ytick={0,20,...,80},
	%
	%
	xlabel=Interactions,
	ylabel=Average Episode Length,
	%
	mark options={scale=0.3},
	cycle list name = custom,
	scaled x ticks=base 10:-5,
    legend style={at={(rel axis cs:1,1)},draw=none,fill=none,name=leg},,
    every x tick scale label/.style={at={(xticklabel cs:0.9,5pt)},yshift=-0.07em,xshift=-4em,left,inner sep=0pt},
    xtick scale label code/.code={\scalebox{0.9}{$(\pgfkeysvalueof{/pgfplots/tick scale binop} 10^{#1})$}}
	]
	\addplot
	table [x=i,y=mean-QL,col sep=comma] 
	{csv/atari/comparison_dg.csv};
	\addplot[name path=top,draw=none,forget plot]
	table [name path=top,x=i,y expr=\thisrow{mean-QL}+\thisrow{std-QL},col sep=comma] 
	{csv/atari/comparison_dg.csv};
	\addplot[name path=bot,draw=none,forget plot]
	table [name path=top,x=i,y expr=\thisrow{mean-QL}-\thisrow{std-QL},col sep=comma] 
	{csv/atari/comparison_dg.csv};
	\addplot [forget plot, draw=none,opacity=0.4,pattern=north east lines,fill=green1!60!black]
	fill between[of=top and bot];
	\addplot
	table [x=i,y=mean-dg,col sep=comma] 
	{csv/atari/comparison_dg.csv};
	\addplot[name path=top,draw=none,forget plot]
	table [name path=top,x=i,y expr=\thisrow{mean-dg}+\thisrow{std-dg},col sep=comma] 
	{csv/atari/comparison_dg.csv};
	\addplot[name path=bot,draw=none,forget plot]
	table [name path=top,x=i,y expr=\thisrow{mean-dg}-\thisrow{std-dg},col sep=comma] 
	{csv/atari/comparison_dg.csv};
	\addplot [forget plot, draw=none,opacity=0.4,pattern=north east lines,fill=orange]
	fill between[of=top and bot];
	\addplot
	table [x=i,y=mean-ours,col sep=comma] 
	{csv/atari/comparison_dg.csv};
	\addplot[name path=top,draw=none,forget plot]
	table [name path=top,x=i,y expr=\thisrow{mean-ours}+\thisrow{std-ours},col sep=comma] 
	{csv/atari/comparison_dg.csv};
	\addplot[name path=bot,draw=none,forget plot]
	table [name path=top,x=i,y expr=\thisrow{mean-ours}-\thisrow{std-ours},col sep=comma] 
	{csv/atari/comparison_dg.csv};
	\addplot [forget plot, draw=none,opacity=0.4,pattern=north east lines,fill=blue1]
	fill between[of=top and bot];
    \addlegendentry{\tiny Q-Learning}
    \addlegendentry{\tiny OPA-PBRS Q-Learning}
    \addlegendentry{\tiny A-PBRS Q-Learning}
	\end{axis}
	\end{tikzpicture}
	\caption{8-rooms. Legend inside the figure itself.}
	\label{fig:comparison_degiacomo}
    \vspace{0.15in}
\end{subfigure}
\quad
\begin{subfigure}{0.45\textwidth}
	%
	\begin{tikzpicture}
	\begin{axis}[
	width=\textwidth,
	height=5cm,
	xmin=0,
	xmax=20000000,
	xtick={0,2000000,...,20000000},
	ymin=0,
	ymax=34,
	ytick={0,8.5,...,34},
	%
	%
	xlabel=Interactions,
	ylabel=Average Return,
	%
	mark options={scale=0.3},
	cycle list name = custom,
	scaled x ticks=base 10:-6,
    every x tick scale label/.style={at={(xticklabel cs:0.9,5pt)},yshift=-0.07em,xshift=-4em,left,inner sep=0pt},
    xtick scale label code/.code={\scalebox{0.9}{$(\pgfkeysvalueof{/pgfplots/tick scale binop} 10^{#1})$}}
	]
	\addplot
	table [x=i,y=mean-DQN,col sep=comma] 
	{csv/atari/freeway.csv};
	\addplot[name path=top,draw=none,forget plot]
	table [name path=top,x=i,y expr=\thisrow{mean-DQN}+\thisrow{std-DQN},col sep=comma] 
	{csv/atari/freeway.csv};
	\addplot[name path=bot,draw=none,forget plot]
	table [name path=top,x=i,y expr=\thisrow{mean-DQN}-\thisrow{std-DQN},col sep=comma] 
	{csv/atari/freeway.csv};
	\addplot [forget plot, draw=none,opacity=0.4,pattern=north east lines,fill=green1!60!black]
	fill between[of=top and bot];
	\addplot
	table [x=i,y=mean-CNN-DQN,col sep=comma] 
	{csv/atari/freeway.csv};
	\addplot[name path=top,draw=none,forget plot]
	table [name path=top,x=i,y expr=\thisrow{mean-CNN-DQN}+\thisrow{std-CNN-DQN},col sep=comma] 
	{csv/atari/freeway.csv};
	\addplot[name path=bot,draw=none,forget plot]
	table [name path=top,x=i,y expr=\thisrow{mean-CNN-DQN}-\thisrow{std-CNN-DQN},col sep=comma] 
	{csv/atari/freeway.csv};
	\addplot [forget plot, draw=none,opacity=0.4,pattern=north east lines,fill=orange]
	fill between[of=top and bot];
	\addplot
	table [x=i,y=mean-DQN-Reshaped,col sep=comma] 
	{csv/atari/freeway.csv};
	\addplot[name path=top,draw=none,forget plot]
	table [name path=top,x=i,y expr=\thisrow{mean-DQN-Reshaped}+\thisrow{std-DQN-Reshaped},col sep=comma] 
	{csv/atari/freeway.csv};
	\addplot[name path=bot,draw=none,forget plot]
	table [name path=top,x=i,y expr=\thisrow{mean-DQN-Reshaped}-\thisrow{std-DQN-Reshaped},col sep=comma] 
	{csv/atari/freeway.csv};
	\addplot [forget plot, draw=none,opacity=0.4,pattern=north east lines,fill=blue1]
	fill between[of=top and bot];	

	\end{axis}
	\end{tikzpicture}
	\caption{Freeway.}
    \vspace{0.15in}
	\label{fig:freeway}
\end{subfigure}
\quad
\begin{subfigure}{0.45\textwidth}
	%
	\begin{tikzpicture}
	\begin{axis}[
	width=\textwidth,
	height=5cm,
	xmin=0,
	xmax=50000000,
	xtick={0,10000000,...,50000000},
	ymin=0,
	ymax=15000,
	ytick={0,5000,...,15000},
	%
	%
	xlabel=Interactions,
	ylabel=Average Return,
	%
	mark options={scale=0.3},
	cycle list name = custom,
	scaled x ticks=base 10:-6,
    every x tick scale label/.style={at={(xticklabel cs:0.9,5pt)},yshift=-0.07em,xshift=-4em,left,inner sep=0pt},
    xtick scale label code/.code={\scalebox{0.9}{$(\pgfkeysvalueof{/pgfplots/tick scale binop} 10^{#1})$}},
    every y tick scale label/.style={at={(yticklabel cs:0.9,5pt)},yshift=-0.5em,xshift=0.22em,left,inner sep=0pt},
    ytick scale label code/.code={\rotatebox{90}{\scalebox{0.9}{$(\pgfkeysvalueof{/pgfplots/tick scale binop} 10^{#1})$}}}
	]
	\addplot
	table [x=i,y=mean-DQN,col sep=comma] 
	{csv/atari/qbert.csv};
	\addplot[name path=top,draw=none,forget plot]
	table [name path=top,x=i,y expr=\thisrow{mean-DQN}+\thisrow{std-DQN},col sep=comma] 
	{csv/atari/qbert.csv};
	\addplot[name path=bot,draw=none,forget plot]
	table [name path=top,x=i,y expr=\thisrow{mean-DQN}-\thisrow{std-DQN},col sep=comma] 
	{csv/atari/qbert.csv};
	\addplot [forget plot, draw=none,opacity=0.4,pattern=north east lines,fill=green1!60!black]
	fill between[of=top and bot];
	\addplot
	table [x=i,y=mean-CNN-DQN,col sep=comma] 
	{csv/atari/qbert.csv};
	\addplot[name path=top,draw=none,forget plot]
	table [name path=top,x=i,y expr=\thisrow{mean-CNN-DQN}+\thisrow{std-CNN-DQN},col sep=comma] 
	{csv/atari/qbert.csv};
	\addplot[name path=bot,draw=none,forget plot]
	table [name path=top,x=i,y expr=\thisrow{mean-CNN-DQN}-\thisrow{std-CNN-DQN},col sep=comma] 
	{csv/atari/qbert.csv};
	\addplot [forget plot, draw=none,opacity=0.4,pattern=north east lines,fill=orange]
	fill between[of=top and bot];
	\addplot
	table [x=i,y=mean-DQN-Preproc-Reshaped,col sep=comma] 
	{csv/atari/qbert.csv};
	\addplot[name path=top,draw=none,forget plot]
	table [name path=top,x=i,y expr=\thisrow{mean-DQN-Preproc-Reshaped}+\thisrow{std-DQN-Preproc-Reshaped},col sep=comma] 
	{csv/atari/qbert.csv};
	\addplot[name path=bot,draw=none,forget plot]
	table [name path=top,x=i,y expr=\thisrow{mean-DQN-Preproc-Reshaped}-\thisrow{std-DQN-Preproc-Reshaped},col sep=comma] 
	{csv/atari/qbert.csv};
	\addplot [forget plot, draw=none,opacity=0.4,pattern=north east lines,fill=blue1]
	fill between[of=top and bot];

	\end{axis}
	\end{tikzpicture}
	\caption{Q*Bert.}
	\label{fig:qbert}
\end{subfigure}
\quad
\begin{subfigure}{0.45\textwidth}
	%
	\begin{tikzpicture}
	\begin{axis}[
	width=\textwidth,
	height=5cm,
	xmin=0,
	xmax=50000000,
	xtick={0,10000000,...,50000000},
	ymin=-100,
	ymax=700,
	ytick={-100,100,...,700},
	%
	%
	xlabel=Interactions,
	ylabel=Average Return,
	%
	mark options={scale=0.3},
	cycle list name = custom,
	scaled x ticks=base 10:-6,
    every x tick scale label/.style={at={(xticklabel cs:0.9,5pt)},yshift=-0.07em,xshift=-4em,left,inner sep=0pt},
    xtick scale label code/.code={\scalebox{0.9}{$(\pgfkeysvalueof{/pgfplots/tick scale binop} 10^{#1})$}}
	]
	\addplot
	table [x=i,y=mean-DQN,col sep=comma] 
	{csv/atari/venture.csv};
	\addplot[name path=top,draw=none,forget plot]
	table [name path=top,x=i,y expr=\thisrow{mean-DQN}+\thisrow{std-DQN},col sep=comma] 
	{csv/atari/venture.csv};
	\addplot[name path=bot,draw=none,forget plot]
	table [name path=top,x=i,y expr=\thisrow{mean-DQN}-\thisrow{std-DQN},col sep=comma] 
	{csv/atari/venture.csv};
	\addplot [forget plot, draw=none,opacity=0.4,pattern=north east lines,fill=green1!60!black]
	fill between[of=top and bot];
	\addplot
	table [x=i,y=mean-CNN-DQN,col sep=comma] 
	{csv/atari/venture.csv};
	\addplot[name path=top,draw=none,forget plot]
	table [name path=top,x=i,y expr=\thisrow{mean-CNN-DQN}+\thisrow{std-CNN-DQN},col sep=comma] 
	{csv/atari/venture.csv};
	\addplot[name path=bot,draw=none,forget plot]
	table [name path=top,x=i,y expr=\thisrow{mean-CNN-DQN}-\thisrow{std-CNN-DQN},col sep=comma] 
	{csv/atari/venture.csv};
	\addplot [forget plot, draw=none,opacity=0.4,pattern=north east lines,fill=orange]
	fill between[of=top and bot];
	\addplot
	table [x=i,y=mean-DQN-Reshaped,col sep=comma] 
	{csv/atari/venture.csv};
	\addplot[name path=top,draw=none,forget plot]
	table [name path=top,x=i,y expr=\thisrow{mean-DQN-Reshaped}+\thisrow{std-DQN-Reshaped},col sep=comma] 
	{csv/atari/venture.csv};
	\addplot[name path=bot,draw=none,forget plot]
	table [name path=top,x=i,y expr=\thisrow{mean-DQN-Reshaped}-\thisrow{std-DQN-Reshaped},col sep=comma] 
	{csv/atari/venture.csv};
	\addplot [forget plot, draw=none,opacity=0.4,pattern=north east lines,fill=blue1]
	fill between[of=top and bot];	

	\end{axis}
	\end{tikzpicture}
	\caption{Venture.}
	\label{fig:venture}
\end{subfigure}
\vspace{0.15in}
\caption{Average return $\pm 1$ standard deviation achieved by the algorithms computed using $10$ independent runs.}
\label{fig:experiments}
\vspace{0.05in}
\end{figure*}

%% file: ablation.tex
\begin{figure*}[!h]
	\centering
	\begin{tikzpicture}
	\begin{axis}[
	width=0.6\textwidth,
	height=5cm,
	xmin=0,
	xmax=50000000,
	xtick={0,10000000,...,50000000},
	ymin=0,
	ymax=15000,
	ytick={0,5000,...,15000},
	%
	%
	xlabel=Interactions,
	ylabel=Average Return,
	%
	mark options={scale=0.3},
	cycle list name = custom,
	scaled x ticks=base 10:-6,
    legend style={at={(rel axis cs:0.18,1)},draw=none,fill=none,name=legend},
    every x tick scale label/.style={at={(xticklabel cs:0.9,5pt)},yshift=-0.07em,xshift=-7.4em,left,inner sep=0pt},
    xtick scale label code/.code={\scalebox{0.9}{$(\pgfkeysvalueof{/pgfplots/tick scale binop} 10^{#1})$}},
    every y tick scale label/.style={at={(yticklabel cs:0.9,5pt)},yshift=-0.5em,xshift=0.22em,left,inner sep=0pt},
    ytick scale label code/.code={\rotatebox{90}{\scalebox{0.9}{$(\pgfkeysvalueof{/pgfplots/tick scale binop} 10^{#1})$}}}
	]
    \addplot
	table [x=i,y=mean-DQN,col sep=comma] 
	{csv/atari/qbert.csv};
	\addplot[name path=top,draw=none,forget plot]
	table [name path=top,x=i,y expr=\thisrow{mean-DQN}+\thisrow{std-DQN},col sep=comma] 
	{csv/atari/qbert.csv};
	\addplot[name path=bot,draw=none,forget plot]
	table [name path=top,x=i,y expr=\thisrow{mean-DQN}-\thisrow{std-DQN},col sep=comma] 
	{csv/atari/qbert.csv};
	\addplot [forget plot, draw=none,opacity=0.4,pattern=north east lines,fill=green1!60!black]
	fill between[of=top and bot];
	\addplot
	table [x=i,y=mean-DQN-Reshaped,col sep=comma] 
	{csv/atari/qbert.csv};
	\addplot[name path=top,draw=none,forget plot]
	table [name path=top,x=i,y expr=\thisrow{mean-DQN-Reshaped}+\thisrow{std-DQN-Reshaped},col sep=comma] 
	{csv/atari/qbert.csv};
	\addplot[name path=bot,draw=none,forget plot]
	table [name path=top,x=i,y expr=\thisrow{mean-DQN-Reshaped}-\thisrow{std-DQN-Reshaped},col sep=comma] 
	{csv/atari/qbert.csv};
	\addplot [forget plot, draw=none,opacity=0.4,pattern=north east lines,fill=orange]
	fill between[of=top and bot];
 	\addplot
	table [x=i,y=mean-DQN-Preproc-Reshaped,col sep=comma] 
	{csv/atari/qbert.csv};
	\addplot[name path=top,draw=none,forget plot]
	table [name path=top,x=i,y expr=\thisrow{mean-DQN-Preproc-Reshaped}+\thisrow{std-DQN-Preproc-Reshaped},col sep=comma] 
	{csv/atari/qbert.csv};
	\addplot[name path=bot,draw=none,forget plot]
	table [name path=top,x=i,y expr=\thisrow{mean-DQN-Preproc-Reshaped}-\thisrow{std-DQN-Preproc-Reshaped},col sep=comma] 
	{csv/atari/qbert.csv};
	\addplot [forget plot, draw=none,opacity=0.4,pattern=north east lines,fill=blue1]
	fill between[of=top and bot];
	\addplot
	table [x=i,y=mean-DQN-Preproc,col sep=comma] 
	{csv/atari/qbert.csv};
	\addplot[name path=top,draw=none,forget plot]
	table [name path=top,x=i,y expr=\thisrow{mean-DQN-Preproc}+\thisrow{std-DQN-Preproc},col sep=comma] 
	{csv/atari/qbert.csv};
	\addplot[name path=bot,draw=none,forget plot]
	table [name path=top,x=i,y expr=\thisrow{mean-DQN-Preproc}-\thisrow{std-DQN-Preproc},col sep=comma] 
	{csv/atari/qbert.csv};
	\addplot [forget plot, draw=none,opacity=0.4,pattern=north east lines,fill=violet]
	fill between[of=top and bot];
    \addlegendentry{NR}
 	\addlegendentry{R}
    \addlegendentry{R+P}
    \addlegendentry{P}
    \end{axis}

	\end{tikzpicture}

\vspace{0.15in}
\caption{Average return achieved by simple fully-connected DQN without Reshaping (NR), with Reshaping (R), with Pre-processing (P), with Reshaping and Pre-processing (R+P) $\pm 1$ standard deviation computed using $10$ independent runs in \textit{Q*Bert}.}
\label{fig:qbert_ablation}
\end{figure*}